\documentclass{article}

\usepackage{microtype}
\usepackage{graphicx}
\usepackage{subcaption}
\usepackage{booktabs} 

\usepackage{hyperref}

\usepackage[preprint]{icml2026}

\usepackage{amsmath}
\usepackage{amssymb}
\usepackage{mathtools}
\usepackage{amsthm}

\usepackage[capitalize,noabbrev]{cleveref}

\theoremstyle{plain}
\newtheorem{theorem}{Theorem}[section]

\newtheorem{lemma}[theorem]{Lemma}

\theoremstyle{definition}

\newtheorem{assumption}[theorem]{Assumption}
\theoremstyle{remark}

\usepackage[textsize=tiny]{todonotes}

\usepackage{multirow}
\theoremstyle{plain}
\usepackage{rotating}

\theoremstyle{remark}

\DeclareMathOperator*{\argmin}{arg\,min}

\usepackage{dsfont}
\newcommand{\X}{X} 
\newcommand{\XTilde}{\widetilde{\X}}
\newcommand{\x}{x}
\newcommand{\xTilde}{\widetilde{\x}}

\global\long\def\esp{\mathbb{E}}%
\global\long\def\P{\mathbb{P}}%
\newcommand{\etaStar}{\eta^\star}

\newcommand{\etaStarTilde}{\etaStar}

\newcommand{\g}{g}
\newcommand{\gStar}{\g^\star}
\newcommand{\gStarTilde}{\gStar}

\newcommand{\I}{\mathds{1}}

\usepackage{pifont}

\icmltitlerunning{When Pattern-by-Pattern Works: Theoretical and Empirical Insights for Logistic Models with Missing Values}

\begin{document}

\twocolumn[
  \icmltitle{When Pattern-by-Pattern Works: Theoretical and Empirical Insights for Logistic Models with Missing Values}



  \icmlsetsymbol{equal}{*}

  \begin{icmlauthorlist}
    \icmlauthor{Christophe Muller}{ox}
    \icmlauthor{Erwan Scornet}{sorbonne}
    \icmlauthor{Julie Josse}{inria}
  \end{icmlauthorlist}

  \icmlaffiliation{ox}{Department of Statistics, University of Oxford, Oxford, United Kingdom}
  \icmlaffiliation{inria}{INRIA PreMeDICaL, Institut Desbrest d’Épidémiologie et de Santé Publique (Idesp), University of Montpellier, Montpellier, France}
  \icmlaffiliation{sorbonne}{Sorbonne Université, Université Paris Cité, CNRS, Laboratoire de Probabilités, Statistique et Modélisation (LPSM), F-75005 Paris, France}

  \icmlcorrespondingauthor{Christophe Muller}{christophe.muller@stats.ox.ac.uk}

  \icmlkeywords{Machine Learning, ICML, Missing values, MCAR, MAR, MNAR, Logistic Models, Consistency}

  \vskip 0.3in
]

\printAffiliationsAndNotice{}

\begin{abstract}
    Predicting with missing inputs challenges even parametric models, as parameter estimation alone is insufficient for prediction on incomplete data. While several works study prediction in linear models, we focus on logistic models, where optimal predictors lack closed-form expressions.
    We prove that a Pattern-by-Pattern strategy (PbP), which learns one logistic model per missingness pattern, accurately approximates Bayes probabilities under a Gaussian Pattern Mixture Model (GPMM). Crucially, this result holds across standard missing data scenarios (MCAR and MAR) and, notably, in Missing Not at Random (MNAR) settings where standard methods often fail.
    Empirically, we compare PbP against imputation and EM methods across classification, probability estimation, calibration, and inference.  Our analysis provides a comprehensive view of logistic regression with missing values. It reveals that mean imputation can be used as baseline for low sample sizes and PbP for large sample sizes, as both methods are fast to train and may have good performances in some settings. The best performances are achieved by non-linear multiple iterative imputation techniques that include the response label (Random Forest MICE with response), which are more computationally expensive.
\end{abstract}

\section{Introduction}
Missing data is a common challenge in supervised learning tasks, where the goal is to predict an outcome variable based on a set of input features. In real-world datasets, missing values often arise due to various reasons such as measurement errors, data corruption, or non-response in surveys. Ignoring missing values can lead to biased models and reduced predictive accuracy, while improper handling can result in misleading conclusions.

Missing data are commonly classified into three categories \citep{rubin1976inference} based on how missingness depends on the data: Missing Completely at Random (MCAR),  Missing at Random (MAR) or Missing Not at Random (MNAR). Building on these definitions, a vast literature focuses on the estimation of model parameters and their distribution in the presence of missing values   \citep[see, e.g.][]{inference2, inference1, inference3}. For instance, \citet{little1992regression,Jones1996IndicatorAS,robins1994estimation} provide methods for parameter estimation in linear models. 

Logistic regression is one of the most used binary classification methods, which often serves as a baseline for this type of problems \citep[see, e.g.,][]{hosmer2013applied}. 
Unfortunately, there exist only a few methods to estimate the coefficients of logistic models in the presence of missing values. This is partly related to the fact that no closed-form expression exists for the coefficients, contrary to linear models with Gaussian features \citep[see, e.g.,][]{le2020neumiss}. 
A naive method to handle missing values is the \textit{complete case}, which consists of applying a logistic regression on complete samples only. This approach cannot be used to predict on inputs with missing values as estimating logistic coefficients is not enough to predict on data with missing values.

The most prevalent strategy consists of imputing the missing values in a first step, and then applying a logistic regression on the imputed dataset. \textit{Constant imputation} is the simplest approach, which replaces missing values by the mean/median/mode or other statistics. However, \citet{lobo2024primerlinearclassificationmissing} proved that constant imputation in a logistic model is not Bayes optimal (i.e., leads to inconsistent estimates). 
Other more powerful imputations can be applied before the logistic regression, such as the Multivariate Imputation by Chained Equations \citep[MICE,][]{van2011mice} which models and imputes iteratively each variable conditionally on the others. MICE is a widely used and flexible imputation method which can handle complex relationships in the data. This makes it a powerful precursor to downstream tasks like the logistic regression.
\textit{Stochastic Approximation Expectation Maximization (SAEM)} algorithm, introduced by \citet{jiang2020logistic}, is a generative model which assumes a logistic regression with Gaussian covariates and infers its parameters via a stochastic EM scheme. This assumption may be restrictive in practice. 

An additional difficulty of predicting in a linear/logistic model comes from the variety of possible models per missing pattern. Assuming a linear model on the complete data does not result in linear models for all missing patterns: it depends on the type of links between input components. For linear models, \citet{le2020neumiss, ayme2022near} established sufficient conditions on input data and missing mechanisms for each model on each missing pattern to be linear. Under such assumptions, a viable strategy consists in building one linear model per missingness pattern, the so-called Pattern-by-Pattern (PbP) strategy. Unfortunately, such assumptions have not yet been found for logistic models. More dramatically, \citet{lobo2024primerlinearclassificationmissing} proved that it is impossible to obtain a logistic submodel from a logistic model on complete data (under MCAR and independent inputs). 
This misspecification therefore suggests that the PbP strategy is inconsistent. However, we observe good predictive performance of PbP in practice therefore motivating the central question of our work: in which settings can PbP provide an effective predictive strategy? 

\paragraph{Contributions} 
We consider Gaussian Pattern Mixture Models (GPMM), where data within each missingness pattern follows a distinct multivariate Gaussian distribution.
Under the GPMM assumption, we show that a Probit model on the complete inputs leads to Probit models on each missing pattern. This result is the first to exhibit a classification model (Probit), which remains well-specified on each missing data pattern. 
Under the same assumptions but replacing the Probit model by a logistic model, we show that PbP logistic regression can closely approximate the Bayes classifier.
This is our main theoretical contribution, which explains why PbP logistic regression may work well in practice. This result holds for some MCAR, MAR and MNAR settings. We further show that real-world missingness is often highly concentrated in a few patterns, mitigating the ``curse of dimensionality'' and ensuring the practical feasibility of PbP.
Then, we conduct a comprehensive empirical comparison across diverse scenarios and rigorously evaluate them using four complementary metrics (classification, probability estimation, calibration, and parameter inference). Our analysis reveals that mean imputation can be used as baseline for low sample sizes. Improved performance can be obtained for small sample sizes using a nonlinear multiple iterative imputation technique that incorporates the labels. For large sample sizes, \texttt{PbP} is the most promising method for GPMM, but we recommend the nonlinear iterative technique in the presence of non-linear features. More details can be found in the experimental section. 

\paragraph{Outline}  
In Section~\ref{sec:pbp_def}, we formalize the problem, define the missingness setting, and introduce the Pattern-by-Pattern (PbP) procedure. In Section~\ref{sec:theory}, we develop our theoretical results on PbP Probit and logistic models. Section~\ref{sec:simu_method} details the simulation setup and evaluation metrics, while results are presented in Sections~\ref{sec:simu_results}. Finally, \Cref{sec:real_data} addresses the curse of dimensionality of PbP and presents empirical results on real datasets.

\paragraph{Related work}
\citet{josse2024consistency}  showed that, under MAR, consistent predictors can be obtained either by imputing missing test values via multiple imputation or by applying constant imputation before training, provided the learner is universally consistent (e.g., kernel methods, nearest neighbors). This result was further extended by \citet{morvan2021s} to any type of missing data. 
\citet{le2020linear} examined what appeared to be a simple case: linear regression with missing data. Their analysis revealed surprising complexity, proving that even when the data-generating process is linear, the optimal predictor may become non-linear in the presence of missing values. This fundamental result highlights that missing values fundamentally change the nature of the prediction problem.
Logistic regression introduces additional challenges. While PbP is a valid approach in linear models with independent MCAR Gaussian inputs \citep{le2020neumiss}, this is not the case for logistic models \citep{lobo2024primerlinearclassificationmissing}.
This result underscores a key insight: even in seemingly simple settings, standard linear methods fail to capture the true conditional probability structure when data is missing. In high dimensions with MCAR data, 
\citet{verchand2024highdimensionallogisticregressionmissing} show that imputation by zero followed by Ridge-regularized logistic regression can attain Bayes performance in classification.

\section{Problem Setting}
\label{sec:pbp_def}

In this paper, we analyze the problem of binary classification in the presence of missing data. We thus consider two random variables $X$ and $Y$, where $X = (X_1, \hdots, X_d) \in \mathds{R}^d$ 
is the input vector and  $Y \in \{0,1\}$ the binary response. 

\paragraph{Missing data} We assume that some components of $X$ may be missing and introduce a missingness indicator vector $M \in \{0,1\}^d$, such that, for all $1 \leq j \leq d$,
\begin{equation} 
M_j = 
\begin{cases} 
    1 & \text{if } X_j \text{ is not observed (\texttt{NA})} \\
    0 & \text{if } X_j \text{ is observed} 
\end{cases} 
\end{equation} 
We define the incomplete random vector $\widetilde{X} \in (\mathds{R} \cup \{\texttt{NA}\})^d$ as $\widetilde{X}_j = X_j$ if $M_j = 0$ and $\widetilde{X}_j = \texttt{NA}$ otherwise. For all $J \subset \{1, \hdots, d\}$, we let $X_J$ be the subvector of $X$ whose components are indexed by $J$. To refer to the observed and missing components of $X$, we define the indices of observed entries $obs(m) = \{j \in \{1,\dots,d\} : m_j = 0\}$ and that of missing entries  $mis(m) = \{j \in \{1,\dots,d\} : m_j = 1\}$. Consequently, $X_{obs(m)}$ is the observed part of $X$ given $M=m$ and $X_{mis(m)}$ is the missing part of $X$ given $M=m$.

\paragraph{Supervised learning }

In a binary classification problem with complete input, we want to find the minimizer of $g \mapsto \esp[\I_{Y \neq g(X)}]$, called the Bayes predictor (for the $0-1$ loss) given by $x \mapsto \I_{\esp[Y|X=x] \geq 0.5}$. In presence of incomplete input, we are interested instead in solving $\gStarTilde \in \argmin_g \esp[\I_{Y \neq g(\XTilde)}],$
whose solution is given by $\gStarTilde(\xTilde) = \mathds{1}_{\etaStarTilde(\xTilde) \geq 0.5}$, where  $\etaStarTilde(\xTilde) = \esp[Y |\XTilde = \xTilde]$.
Note that the Bayes probability $\etaStarTilde$ is interesting on its own as it provides richer insight—like the probability of a disease in medical diagnosis—enabling more nuanced decision-making.

\paragraph{Pattern-by-Pattern approaches} The Bayes probability $\etaStarTilde(\widetilde{X})$ can be decomposed by missing patterns:
\begin{equation}
\label{eq:bayes_prob}
    \etaStarTilde(\XTilde) = \sum_{m \in \{0,1\}^d} \eta^\star_m(X_{obs(m)}) \cdot \mathds{1}(M=m)
\end{equation}
where $\eta^\star_m(X_{obs(m)}) = \esp[Y|X_{obs(m)},M=m]$ is the Bayes probability on missing pattern $m$.
The Pattern-by-Pattern (PbP) strategy leverages this decomposition by fitting, for each missing pattern $m \in \{0,1\}^d$, a supervised learning model (e.g. logistic regression) on the sub-sample of observations with missing pattern $M=m$. Letting $\widehat{\eta}_m$ be the resulting predictor on pattern $m$, the overall PbP estimator is defined by
\begin{equation}
\label{eq:h_pbp}
    \widehat{\eta}^{\mathrm{PbP}}(\XTilde) = \sum_{m \in \{0,1\}^d} \widehat{\eta}_m(X_{obs(m)}) \cdot \I_{M = m}.
\end{equation}

\section{Convergence Properties of the PbP Estimator}
\label{sec:theory}

\subsection{Well-Specified Case: The Probit Model}
Generalized linear models \citep[][]{nelder1972generalized} are a wide class of parametric models which can be used to model a variety of outputs (binary, continuous, counts...) via a variety of methods. They include the well-known linear and logistic models and  the Probit model, which we study below. 

\begin{assumption}[Probit model]\label{ass:probit_model_complete}
Let $\Phi(t) = (2\pi)^{-1/2} \int_{-\infty}^t  e^{-t^2/2} \textrm{d}t$. There exist $\beta_0^\star, \hdots, \beta_d^\star \in \mathds{R}$ such that
the distribution of $Y$ given the complete input $X$ satisfies $\mathds{P}[Y=1|X] = \Phi ( \beta_0^\star + \sum_{j=1}^d \beta_j^\star X_j)$.
\end{assumption}

We also need to make some assumptions on the missing data mechanism: data follow the Gaussian Pattern Mixture Model (GPMM) described below. 

\begin{assumption}[GPMM]\label{ass:mcar}
For all $m \in \{0,1\}^d$, there exist $\mu_m, \Sigma_m$ such that $X|M=m \sim \mathcal{N}(\mu_m, \Sigma_m)$.
\end{assumption}

GPMM encompasses some MCAR, MAR, and MNAR mechanisms \citep[see, e.g.,][]{ayme2022near}. In particular, if all $\mu_m$ and all $\Sigma_m$ are equal, we obtain a MCAR mechanism with Gaussian inputs. For any set $S$, we denote by $|S|$ its cardinality.

\begin{theorem}
\label{th:probit_model}
Grant \Cref{ass:probit_model_complete} and \ref{ass:mcar}. Then, for all $m \in \{0,1\}^d$, 
the Bayes probabilities on pattern $m$ satisfies, for all $x \in \mathds{R}^{|obs(m)|}$,
\begin{align*}
   \eta^\star_m(x)= \Phi\left( \frac{ \alpha_{0,m}  +  \alpha_m^\top x  }{\sqrt{1+  \tilde{\sigma}^2_m}}\right),
\end{align*}
where, letting $O_m = \Sigma_{m, obs(m), obs(m)}$,
\begin{align*}
     \alpha_{0,m}  & = \beta_0^\star + (\beta^{\star}_{mis(m)})^\top\mu_{m, mis(m)}\\&  \quad  - (\beta^{\star}_{mis(m)})^\top \Sigma_{m, mis(m), obs(m)} O_m^{-1}  \mu_{m, obs(m)}\\
     \alpha_m & = \beta_{obs(m)}^\star +  O_m^{-1} \Sigma_{m, obs(m), mis(m)} \beta^{\star}_{mis(m)} \\
     \tilde{\sigma}^2_m &  = (\beta^{\star}_{mis(m)})^\top \widetilde{\Sigma}_m \beta^{\star}_{mis(m)} \\
     \widetilde{\Sigma}_m &  = \Sigma_{m, mis(m), mis(m)} \\
    & \quad - \Sigma_{m, mis(m), obs(m)} O_m^{-1}  \Sigma_{m, obs(m), mis(m)}.
\end{align*}

\end{theorem}
The proof of \Cref{th:probit_model} can be found in \Cref{app:proofs_1}.
According to \Cref{th:probit_model}, if we assume that the data follow a GPMM with a Probit model, then each Pattern-by-Pattern predictor follows a Probit model. Therefore, under these assumptions, estimating each Bayes predictor via a Probit model leads to consistent estimators. 

Inspection of the proof of  \Cref{th:probit_model} reveals that 
\begin{align}
\tilde{\sigma}^2_m  & = \textrm{Var}[ X_{mis(m)}^\top \beta^{\star}_{mis(m)}  | X_{obs(m)} = x ]\\
    \alpha_{0,m} + \alpha_m^\top x  & = \mathds{E}[ \beta_0^\star + \sum_{j=1}^d \beta_j^\star X_j | X_{obs(m)}=x].\end{align}
The properties of the function $\Phi$, combined with the Gaussian nature of the input allows each PbP Bayes predictor (which are  conditional expectations of the function $\Phi$ given $X_{obs(m)}$) to be written as the function $\Phi$ applied to conditional expectations. Thus, it is not surprising that the expressions for $\alpha_{0,m}$ and $\alpha_m$ are the same as those obtained in linear models with missing values. More precisely, grant \Cref{ass:probit_model_complete} replacing $\Phi$ by the identity function, thus leading to a linear model. Then, \citet{le2020neumiss} and \citet{ayme2022near} prove that 
\begin{align}
   \mathds{E}[Y | X_{obs(m)} = x, M=m]  =   \alpha_{0,m}  +  \alpha_m^\top x.
\end{align}
The non-linearity of the function $\Phi$ (inherent to the binary nature of $Y$) introduces a renormalizing factor $(1 + \tilde{\sigma}^2_m)^{-1}$. This corrective factor depends only on $\beta^{\star}$ and the covariance matrices $\Sigma_m$. 

In order to interpret PbP predictor in \Cref{th:probit_model}, let us consider the simple case where all components of $X$ are independent, that is $\Sigma = \textrm{diag}(\sigma_1^2, \hdots, \sigma_d^2)$. In this case, we have $ \alpha_{0,m}  = \beta_0^\star + (\beta^{\star}_{mis(m)})^\top\mu_{mis(m)}$, $\alpha_m = \beta_{obs(m)}^\star$ and  
\begin{align*}
    \tilde{\sigma}^2_m &  = \sum_{j \in mis(m)} (\sigma_j \beta^{\star}_{j, mis(m)})^2.
\end{align*}
Thus the corrective factor $\tilde{\sigma}^2_m$ is all the more important for missing patterns that involve a high number of missing components, which are important in the predictive model (large $\beta^{\star}$) or with a large variance $\sigma_j^2$. Obviously, if for a given pattern $m$, all missing components are not involved in the predictive model ($\beta^{\star}_{mis(m)}=0$) then we retrieve the original Probit model with $\alpha_{0,m} = \beta_0^{\star}$ and $\alpha_m = \beta^{\star}_{obs(m)}$.

Note that if no relation is known between the different elements $(\Sigma_m, \mu_m)_{m \in \{0,1\}^d}$, one may be forced to estimate all these parameters to obtain all PbP predictors. 
In such settings, though computationally costly, PbP strategies appear to be among the few consistent strategies. 

\subsection{Logistic with Gaussian Covariates approximates Bayes predictor}
\label{sec:theoretical_approx_gaussian_covariates_pbp}

\begin{assumption}[Logistic model]\label{ass:logistic_model_complete}
Let $\sigma (t)= 1/(1 + e^{-t})$. There exist $\beta_0^\star, \hdots, \beta_d^\star \in \mathds{R}$ such that
the distribution of $Y$ given the complete input $X$ satisfies $\mathds{P}[Y=1|X] = \sigma ( \beta_0^\star + \sum_{j=1}^d \beta_j^\star X_j)$.
\end{assumption}

Under a MCAR model, \Cref{ass:logistic_model_complete} and Gaussian independent covariates, \citet{lobo2024primerlinearclassificationmissing} prove that the Bayes predictor on each missing pattern is not a logistic model. However, \Cref{th:approx_logistic} below illustrates that PbP logistic regression is a good approximation of the Bayes predictor when input data are Gaussian.

\begin{theorem}
\label{th:approx_logistic}
Grant \Cref{ass:mcar} and \ref{ass:logistic_model_complete}. Then, for all $m \in \{0,1\}^d$, 
the Bayes predictor on pattern $m$ satisfies, for all $x \in \mathds{R}^{|obs(m)|}$,
\begin{align*}
    & \left| 
    \eta^\star_m(x) - \sigma\left( \frac{ \alpha_{0,m}  +  \alpha_m^\top x  }{\sqrt{1+ (\pi/8) \tilde{\sigma}^2_m}}\right) \right| \leq 2 \|\varepsilon\|_{\infty},
\end{align*}
where $\varepsilon(t) = \Phi(t) - \sigma(t \sqrt{8/\pi})$, and the expression of $\alpha_{0,m}, \alpha_m, \tilde{\sigma}^2_m$ and $\widetilde{\Sigma}$ are given in \Cref{th:probit_model}.
\end{theorem}

The proof of \Cref{th:approx_logistic} can be found in \Cref{app:proofs_2}. \Cref{th:approx_logistic}  establishes that each PbP Bayes predictor is close to a logistic function, assuming that complete data follows a logistic model, with a GPMM. Indeed, the scaled sigmoid function $\sigma(t\sqrt{8/\pi})$ is close to the probit function: numerical simulations give $\|\varepsilon\|_{\infty} \simeq 0.018$. Thus, while the results of \citet{lobo2024primerlinearclassificationmissing} prove that, in our setting, each PbP Bayes predictor is not exactly logistic, \Cref{th:approx_logistic} highlight that each one is close to a logistic function, with an approximation error of at most $\|\varepsilon\|_{\infty} \simeq 0.018$.

\subsection{Illustration in a 2-Dimensional Setting}
\label{sec:illustration_th}

To illustrate our results, consider $X \sim \mathcal{N}((1.5, 0), \textrm{diag}(5, s^2))$. The binary outcome $Y$ follows the logistic model $\mathds{P}[Y=1|X]  = \sigma(X_1 + X_2)$.
Assume $X_1$ is always observed, while $X_2$ is always missing. One could be tempted to estimate probabilities by applying the logistic model for complete data, replacing  $X_2$ by its expectation. Due to the non-linearity of the sigmoid function, this gives poor results: the yellow curve in \Cref{fig:expF_theory} (left) is not a good approximation of the Bayes probabilities (pink curve). 
Besides, the Bayes probabilities do not exactly follow a logistic model: the blue dotted curve (best logistic approximation) departs from the pink curve for $x_1 \simeq \pm 8$ (see \Cref{fig:logit_normal_exponential} in Appendix for more details). This is in line with   
\citet{lobo2024primerlinearclassificationmissing}, and with our result (\Cref{th:approx_logistic}) stating that the Bayes probabilities are not equal but close to a logistic model. 

\begin{figure}[ht!]
\centering
\includegraphics[width=0.47\textwidth]{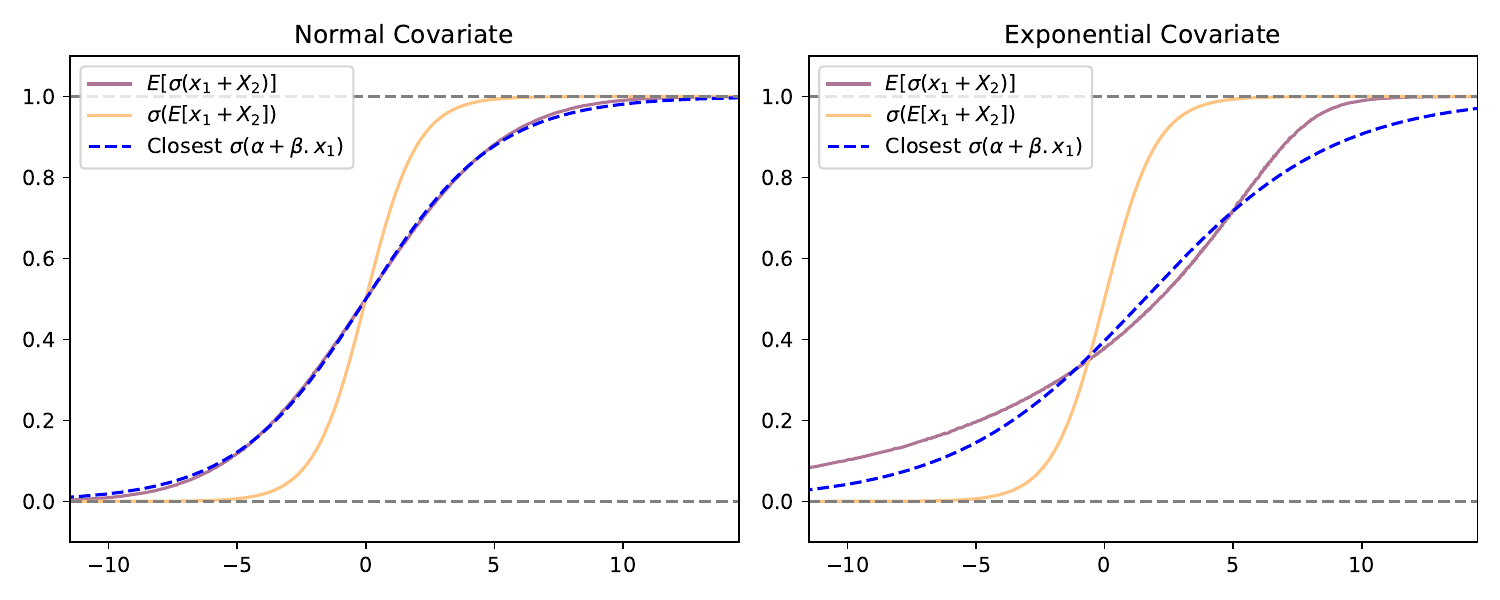}
\caption{Left: When $X_2 \sim \mathcal{N}(0,s^2)$, the best logistic approximation $\sigma(\alpha + \beta x_1)$ closely matches the true probabilities $\esp[\sigma(x_1 + X_2)]$, consistent with Theorem~\ref{th:approx_logistic}. Right: When $X_2 \sim \text{Exp}(\lambda)-\lambda$, the best logistic approximation deviates significantly from the true probabilities. $s^2 \approx 3.83$ and $\lambda \approx 7.63$ are chosen to maximize the deviation from logistic approximation.}
\label{fig:expF_theory}
\end{figure}

To test the role of normality, we repeat the analysis by generating $X_2$, independently of $X_1$ with $X_2 \sim \text{Exp}(\lambda)-\lambda$. As shown in \Cref{fig:expF_theory} (right), the approximation deteriorates, underscoring the importance of the Gaussian assumption in \Cref{th:approx_logistic}.

\section{Methods and Metrics}
\label{sec:simu_method}

\subsection{Procedures}

We outline several procedures to address missing values in the covariates. Unless specified otherwise, a logistic regression is trained on the imputed dataset. 

\textbf{Complete Case (CC)} consists in 
    excluding rows with missing values. 

\textbf{Constant Imputation (\textit{C}.IMP)} substitutes all missing values with a fixed constant \textit{C} (e.g., imputation by $0.5$ is \texttt{05.IMP}). A particular case is mean imputation (\texttt{Mean.IMP}) where any missing value is replaced by its feature mean.  

\textbf{$K$ Imputations by MICE (MICE.\textit{K}.IMP)} employs iterative regression models to  impute missing values for each feature, conditioned on the others. It is implemented using the \texttt{MICE} package in $\mathbf{R}$ \citep{van2011mice}, with its default predictive mean matching method (PMM), and a number of multiple imputed datasets equal to $K$. Logistic regressions applied on each imputed dataset produce parameter estimators $\hat{\beta}^{(k)}$ and probability estimators $\widehat{\eta}^{(k)}(x)$, for all $k \in \{1, \hdots, K\}$, and all $x \in \mathbb{R}^d$. These estimators are then  aggregated using \citet{rubin2018multiple}'s rule to provide the final estimators
    \begin{equation}
        \hat{\beta} = \frac{1}{K}\sum_{k=1}^K \hat{\beta}^{(k)} \quad \textrm{and} \quad 
        \widehat{\eta}(x) = \frac{1}{K}\sum_{k=1}^K \widehat{\eta}^{(k)}(x),
    \end{equation}
    A variant, \textbf{\texttt{MICE.\textit{K}.Y.IMP}}, incorporates the response variable ($Y$) as a covariate during the imputation of features ($X$) in the training phase \citep{d2024behind}. This allows $Y$ to inform the imputation of $X$. In the test phase, $Y$ is treated as unobserved, and its estimated probability is derived from the logistic regression on the imputed $X$ only. 
%
   We also consider variants where the inner distributional regressor is a random forest (RF). These methods are denoted \texttt{MICE.RF.\textit{K}.IMP} and \texttt{MICE.RF.\textit{K}.Y.IMP}.

 \textbf{Stochastic Approximation EM} \citep[SAEM,][]{jiang2020logistic} jointly models covariates and response variables assuming normality for covariates and a logistic response. Unlike imputation-then-regression approaches (like \textbf{\textit{C}.IMP} and \textbf{MICE.\textit{K}.IMP}), SAEM directly estimates the logistic regression parameters in the presence of missing data, performing the EM steps using stochastic approximation. The prediction step of SAEM is then based on a multiple imputation procedure using estimated parameters.

We used the \textit{misaem} R-package's implementation of the SAEM algorithm for our simulations. This package was updated as part of this work, and to ensure broader accessibility of the algorithm, we have also implemented it in Python and Julia\footnote{The Python and Julia packages are available at \url{https://pypi.org/project/misaem/} and \url{https://juliahub.com/ui/Packages/General/LogisticSAEM}, respectively.}.

\textbf{Pattern by Pattern (PbP)} fits a different logistic model on each missingness pattern, using the data of that specific pattern only. We observed that applying regularization yields stability gains in low-sample regimes. For clarity reasons, we choose to focus on standard (unregularized) logistic regression in this work.

We also consider variants of these procedures, where the missingness mask $M$ is incorporated as additional covariates of the logistic model, denoted by appending  $M$ to their names (e.g., \texttt{Mean.IMP.M}). Adding $M$ as an extra feature was shown to enhance predictive performance when the mask is informative \citep{van2023missing}. In the case of MICE imputations, we also consider incorporating the mask $M$ \emph{before} the imputation, denoted by \texttt{MICE.\textit{K}.M.IMP} and \texttt{MICE.\textit{K}.Y.M.IMP}.

Among these methods, \texttt{CC} is impractical for test sets with missing values, as it cannot generate either predictions or class probabilities. Constant and linear imputation methods are not Bayes optimal \citep[see][] {lobo2024primerlinearclassificationmissing}; they always lead to inconsistent procedures. 
 We expect similar issues to affect the \texttt{MICE.\textit{K}.IMP} procedure when \textit{K} is small, even if this more complex imputation is not explicitly covered by their proposition. On the contrary, we expect \texttt{SAEM} and \texttt{MICE.\textit{K}.IMP} for large $K$ to converge to the Bayes probabilities when the covariates are normally distributed, as their design resembles the true Bayes probabilities generation. A similar behavior is expected for \texttt{PbP} which produces probabilities close to the Bayes probabilities under the Gaussian assumptions (\Cref{th:approx_logistic}).

\subsection{Evaluation Methods}
\label{subsec:evaluation}

We evaluate each method from four complementary perspectives: classification, probability estimation, calibration, and inference, all described below. We evaluate the \textbf{classification} performance via the excess misclassification rate (percentage of incorrect prediction of the method minus that of the Bayes classifier). As probabilities may provide more nuanced information than the binary prediction, we compare the estimated probabilities (\textbf{Probability Estimation}) with the Bayes probabilities via the Mean Absolute Error (MAE, see  \Cref{app:howtocomputepstar}). A low MAE usually implies a low misclassification rate.  
\textbf{Calibration} measures the reliability of the predicted probabilities, specifically, whether a predicted probability $\eta$ corresponds to an actual outcome frequency $\mathbb{P}(Y=1 \mid \hat{\eta} = \eta) = \eta$. Following \citet{dimitriadis2021stable}, we focus on the \emph{Miscalibration (MCB)} component of their decomposition of the Brier score (see \Cref{sec:risk_measures_appendix}), and compute its difference with the MCB of the Bayes classifier. Finally,  we compute the \emph{Mean Squared Error} (MSE) between the estimated and true coefficients (\textbf{Inference}) where, to be fair, we exclude the intercept (including the mask in the prediction step leads to more than one constant coefficient). 

\section{Simulations}
\label{sec:simu_results}

In this section, we investigate the performance of the selected procedures on various simulated models\footnote{The code for all simulations and analyses presented in this paper is publicly available at \url{https://github.com/ChristopheMuller/logistic_with_NAs}.}. We denote $\mathcal{B}$ the Bernoulli distribution and $\mathcal{U}$ the uniform distribution. 

\subsection{Simulation protocol}
\label{sec:methodo_SimA}

For all experiments,  $Y$ follows a logistic model without intercept and with a parameter $\beta^\star$, drawn once for all experiments as $\mathcal{N}(\mathbf{0},I)$. 
Once we have set a distribution for the missing pattern $M$ (see below for details), we first analyze three different GPMM scenarios (\Cref{ass:mcar}).
 The first one is a MCAR model defined as $\Sigma_m = \Sigma$ and $\mu_m = \mu$.
The second variant is inspired by \citet[Example 2.3]{ayme2022near} where some features are always observed, and only their corresponding part of $\mu_m$ and $\Sigma_m$ varies from one pattern to another. This describes a MAR mechanism. Its results are close to the MCAR setting and discussed in Appendix \ref{app:gpmm_mar}.
In the third variant, $\mu_m$ and $\Sigma_m$ are completely different for each pattern: this corresponds to a MNAR setting. \texttt{PbP} is expected to perform well in all cases (see \Cref{sec:theory}). On the other hand, \texttt{SAEM} is expected to perform well only in the first two scenarios (MCAR and MAR) but not in the last one, since it is designed to handle MAR data.

To assess the generalizability of our results, we also consider a fourth experiment with non-Gaussian inputs and MCAR missingness. In this setting, we expect SAEM to perform poorly (as it was designed for Gaussian inputs), similarly to MICE with PMM (which uses linear relations to impute). We also want to assess how \texttt{PbP} behaves in this framework which falls outside that of \Cref{th:approx_logistic}. Finally, we combine the complexity of non-linear covariates with MNAR missingness in a fifth experiment. Its setup and results are discussed in \Cref{app:non_linear_MNAR}. The results are close to the ones of the non-linear MCAR setting, their discussion is thus postponed to \Cref{app:non_linear_MNAR}.

In summary, we will consider in the next sections the following three scenarios, in which the feature dimension is set to $d=5$.

\begin{itemize}
    \item \textbf{(MCAR)} For all $m$, $\mu_m = \mathbf{0}$ and $\Sigma_m =   [\rho^{|i-j|}]_{i,j=1}^d$ with $\rho = 0.65$ (Toeplitz covariance structure).
    All components $M_j$  of the missingness mask are distributed as  independent Bernoulli with parameter $ \P[M_j=1] = 0.25$, and then adjusted via resampling to obtain $\mathds{P}[M = \mathbf{1}] = 0$.

    \item \textbf{(MNAR)}
    For each $m$, we let $\Sigma_m = \sigma_m[\rho_m^{|i-j|}]_{i,j=1}^5$ and we sample $\rho_{m}\sim \mathcal{U}([-1,1])$, $\sigma_m \sim \mathcal{U}([0,1])$ and $\mu_{m}\sim \mathcal{N}(0,0.5I_5)$. The mask $M$ is generated as in the MCAR scenario.  

    \item \textbf{(Non-linear)} We generate $(\X, M)$ as in the MCAR setting, with $\rho = 0.95$. We then apply specific non-linear transformations to $X$ to create the final input $Z$, where $Z_1=X_1, Z_2=X_2$ and for all $j \in \{3, 4, 5\}$, $Z_j = g_j(X_j)$ where $g_j$ are non-linear invertible functions (see \Cref{sec:app_nonlinearfeatures} for details). Thus, we have access to observations distributed as $(Z, M, Y)$, where $Y$ is generated as a logistic model based on $Z$.
    
\end{itemize}

\paragraph{Training and Evaluation} We conduct experiments with varying training set sizes $n \in \{500, 1.000, 5.000, 15.000, 50.000\}$ and a test set of size $15.000$. For each training set size, we create 10 datasets $(X_i, M_i, Y_i)$ via the protocol described above.

\begin{figure*}[ht!]
\centering
\includegraphics[width=0.99\textwidth]{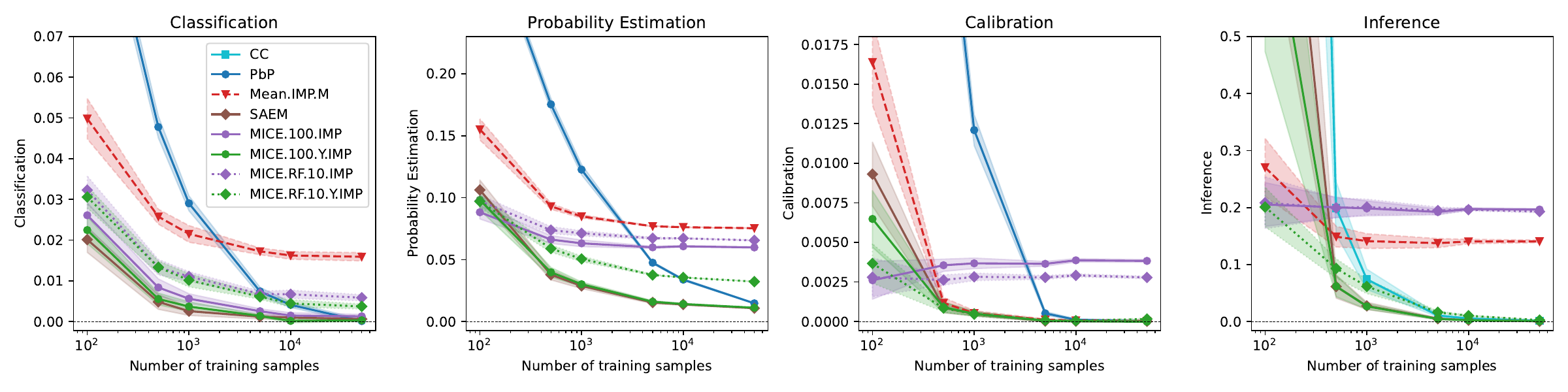}
\vspace{-0.5cm}
\caption{Performances of selected procedures in terms of Misclassification, Calibration, MSE of $\hat{\beta}$ and MAE from Bayes probabilities. Mean/s.e. are computed over 10 replicates of GPMM-MCAR (see \Cref{sec:methodo_SimA}).}
\label{fig:SimA_SelectedProcedures}
\end{figure*}

\subsection{Results for MCAR scenario}
\label{sec:results_MCAR_gaussian}

All procedures  (and their running times) are summarized in Table~\ref{tab:runtimeSimA}. All results are displayed in \Cref{fig:SimA_ALL}. We start by general comments across the metrics described in \Cref{subsec:evaluation}. 

\paragraph{Preliminary results - Constant imputation and MICE}
\Cref{fig:SimA_ALL} (a, b) displays the performance of single imputation procedures. We observe that in terms of constant imputation (\texttt{05.IMP} and \texttt{Mean.IMP}), adding the missingness mask improves the performance of \texttt{05.IMP}, while slightly deteriorating the one of \texttt{Mean.IMP} for small amount of training data.  In fact, similarly to what happens in linear regression \citep[see Proposition 3.1]{le2020linear}, it is straightforward to show that adding the mask amounts to optimizing the imputation constant (here with respect to the logistic loss), leading to \texttt{Mean.IMP.M} and \texttt{05.IMP.M} being equivalent. 
In \Cref{fig:SimA_ALL} (b, c, d, e), we observe that incorporating the mask $M$ in the imputation process of MICE methods has no effect. This was expected: the mask provides no extra information to the imputation model in  MCAR settings. Besides, incorporating the mask as an additional input vector in the logistic model slightly deteriorates the performance of MICE for small sample sizes, across all metrics, since it doubles the number of input features.

\paragraph{Calibration}
\Cref{fig:SimA_SelectedProcedures} aggregates selected method results. \texttt{Mean.IMP.M} exhibits calibration approaching zero, though with poor small training set performance. This trend is more pronounced in \texttt{PbP}, which requires large training sets for convergence. For MICE imputation, good performance is achieved by multiple imputations (with label $Y$).

\paragraph{Parameter estimation}
In terms of parameter inference, the MSE of \texttt{CC}, MICE imputations with labels $Y$ and \texttt{SAEM} vanish when the training set increases, estimating correctly $\beta^\star$. These methods have poor performances for small sample sizes. This was expected for \texttt{CC} as incomplete observations are discarded. 
\texttt{MICE.RF.Y} offers robust performances for small and large training sets. \texttt{PbP} cannot estimate $\beta^\star$ as the method fits one model per missing pattern.

\paragraph{Predictive Tasks} 
We observe that small classification errors are linked to small probability estimation errors. In both metrics, \texttt{SAEM} and MICE with multiple imputations demonstrate superior performance. This was expected for \texttt{SAEM} which is designed to correctly estimate Bayes probabilities in logistic regression with MAR data. We note however that incorporating the label $Y$ is necessary for multiple imputation MICE to approach the Bayes probabilities but not to attain the Bayes classification error. 
The advantages of multiple imputation within \texttt{MICE} procedures are evident; they provide robustness against the inherent imputation variance that can otherwise degrade the performance of single imputation methods. This is illustrated by \texttt{MICE.1.Y.IMP}, which, despite retrieving correctly the parameter $\beta^{\star}$, exhibits comparatively poor performance in predictive tasks. \texttt{MICE.RF} exhibit similar trends as the PMM variants, but its extra complexity comes at a performance cost.
Furthermore, the misclassification risk of Pattern-by-Pattern method approaches the Bayes risk for large training set. Similarly, the \texttt{PbP}'s probabilities are close to the Bayes probabilities, which is in line with \Cref{th:approx_logistic}, given the Gaussian distribution of the covariates in our simulation.

\paragraph{Running time} 
\Cref{tab:runtimeSimA} displays the training times and prediction time of each method. We observe that the training of all single imputation and complete case procedures only require less than 4s (for 50.000 training points). The training times of MICE procedures is linear in the number of imputations. Training the random forest variants is more costly than the default MICE models. \texttt{PbP} is fast to train, though its total cost depends on the number of distinct missingness patterns encountered. \texttt{SAEM} emerges as one of the slowest methods for both training and predicting (648s for 50.000 training points, 11s for 15.000 prediction points).

\begin{figure*}[ht!]
\centering
\includegraphics[width=0.99\textwidth]{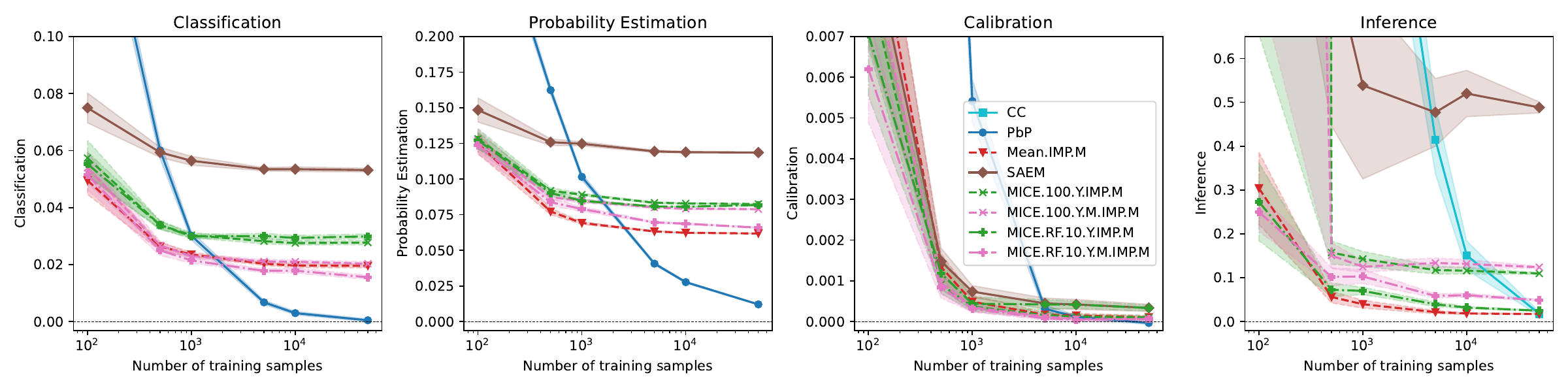}
\vspace{-0.5cm}
\caption{Performances of selected procedures in terms of Misclassification, Calibration, MSE of $\hat{\beta}$ and MAE from Bayes probabilities. Mean/s.e. are computed over 10 replicates of GPMM-MNAR (see \Cref{sec:methodo_SimA}).}
\label{fig:SimG_SelectedProcedures}
\end{figure*}

\subsection{Results for MNAR scenario}
\label{sec:resulst_mnar}

Selected methods for the MNAR setting appear in  \Cref{fig:SimG_SelectedProcedures} (details in \Cref{fig:SimG_ALL}). For MICE imputation methods (PMM and RF), best performance for classification and probability estimation are obtained when the mask $M$ is added in the logistic model, with further gains from adding $M$ and $Y$ in the imputation process.
For calibration and inference, $Y$ is required in the imputation process but $M$ is not. In this MNAR scenario, \texttt{PbP} is the only method seeming to converge to null error for classification, probability estimation and calibration metric. Such good performances are expected from \Cref{th:approx_logistic}. On the contrary, \texttt{SAEM} is the worst method for classification, probability estimation and parameter inference but surprisingly manages to maintain good performance in calibration. A more detailed analysis of several missing patterns reveals different behaviors. In \Cref{fig:SimG_MAE_patterns}, we notice that in the first and fourth missing patterns, the difference in classification between \texttt{PbP} and \texttt{Mean.IMP.M} for $n=50.000$ is $0.05$, much larger than the average $0.02$ displayed in \Cref{fig:SimG_SelectedProcedures}. Thus, good performance of inconsistent methods as \texttt{Mean.IMP.M} may hide poor performances on some specific missing patterns, which may correspond to some specific populations.

\subsection{Results for non-linear features}
\label{sec:additional_simulations_nonlinearfeatures}

\Cref{fig:SimC_ALL} displays the results for non-linear features with MCAR missingness. We see that all MICE variants have similar performances in classification. Adding $Y$ to MICE imputation improves probability estimation. While no method approximates the Bayes probabilities, multiple imputations \texttt{MICE.RF} approach the Bayes risk in classification and perform the best in probability estimation. The other methods (\texttt{MICE}, \texttt{SAEM}, \texttt{PbP}) suffer from the non-linearity of the features, since conditional expectations of missing input given observed inputs are not linear anymore.
Regarding calibration and parameter inference, we obtain similar results as for Gaussian features, with the exception that \texttt{Mean.IMP.M} is here miscalibrated.

\Cref{fig:SimC_MAE_patterns} displays the probability estimation performance when the prediction set has the 5 possible missingness mask with a single value missing. This figure reveals that for Gaussian missing features, one retrieves the same results as in the fully Gaussian simulation (\Cref{sec:results_MCAR_gaussian}). This is anticipated as conditional distribution of the missing data given the observed part is normal, in line with the assumptions of \texttt{SAEM} and \Cref{th:approx_logistic}. However, performance dropped for all methods when non-linear features are missing, especially for non-monotonic transformations.

\subsection{Experimental conclusions}

All experimental results are synthesized in \Cref{tab:misclassification_summary}- \ref{tab:mse_error_summary}. 

\paragraph{Classification performances} For a low sample size, \texttt{Mean.IMP} is a good baseline for GPMM, with a very low training time. In presence of non-linear features, \texttt{MICE.RF.Y} outperforms the other competitors, as it can detect non-linear relations between inputs. For a large sample size, we recommend using \texttt{PbP}, which is among the most efficient strategy for GPMM with a very low training time. For non-linear features, \texttt{MICE.RF.Y} has the best accuracy, but with a large training time. The same conclusions hold for probability estimation. 

\paragraph{Calibration} 
\texttt{Mean.IMP} has a good calibration score for GPMM settings, either in low- or high-sample scenarios. For low sample sizes, best results are obtained for \texttt{MICE.RF.Y} regardless of the data generating mechanism. For large sample sizes, we recommend using \texttt{PbP} which has a low training time and calibration error, closely followed by \texttt{Mean.IMP}. Besides, \texttt{MICE.RF.Y.IMP} has good calibration scores, but with a high training time. 

\paragraph{Parameter estimation} 
For low sample size, \texttt{Mean.IMP} is competitive both in terms of MSE and computation time. Best results in terms of MSE are obtained for all variants of \texttt{MICE.RF} and \texttt{MICE}. Adding either $M$ or $Y$ in \texttt{MICE} drastically degrades its performance. For large sample size,  \texttt{Mean.IMP} is not competitive anymore and \texttt{CC} should be preferred. Adding $Y$ to \texttt{MICE} or \texttt{MICE.RF} improves their performance: while requiring a large training time, both have a low MSE. 

\section{Real Datasets and Curse of Dimensionality}
\label{sec:real_data}
While the theoretical number of missingness patterns grows exponentially (i.e. $2^d$), we argue this worst-case scenario is rare in practice. We analyzed 20 datasets from \textit{R-miss-tastic} \citep{rmisstastic} and from the imputation benchmark from \citet{grzesiak2025needdozensmethodsreal} (see \Cref{app:real_datasets} for selection details and \Cref{tab:real_data_sets} for summary statistics of the data).

In 18 of the 20 datasets, the 10 most frequent patterns cover over 80\% of observations (see \Cref{tab:real_data_sets}). For instance, the \textit{soybean} dataset (d=36) has 68 billion potential patterns, yet only contains $9$ distinct patterns. This concentration suggests that missingness is highly structured, validating the computational feasibility of Pattern-by-Pattern (PbP) approaches in real-world scenarios.

We evaluated a panel of methods across these datasets, with results summarized in \Cref{fig:real_data1} and \ref{fig:real_data2}. While \texttt{PbP} remains competitive in datasets with highly structured missingness, such as \textit{globwarm}, \textit{pedestrian}, and \textit{selfreport} (which exhibit only 2 distinct patterns), the MICE imputation variants offer greater overall reliability. Notably, adding a missingness mask \emph{after} imputation is often detrimental. Surprisingly, despite the complexity of real data often violating the strong normality assumption, \texttt{SAEM} does not systematically underperform; in fact, it achieves the best AUC in several datasets (see \Cref{tab:auc_results}).

\section{Conclusion}

This work explored the challenges and solutions associated with logistic regression in the presence of missing data within covariates. Firstly, we theoretically demonstrated that the \texttt{PbP} strategy can effectively approximate Bayes probabilities when covariates are mixtures of Gaussian. While the number of possible missingness patterns grows exponentially, our analysis of real-world datasets reveals that this ``curse of dimensionality'' is often mitigated in practice, making \texttt{PbP} computationally efficient in most real-data scenarios.
Secondly, we conducted a comprehensive empirical comparison of various strategies for handling missing values. These methods were evaluated across four key aspects: classification, probability estimation, calibration, and parameter inference, with results condensed in \Cref{tab:misclassification_summary} to \ref{tab:mse_error_summary}. We derived a set of practical guidelines for selecting appropriate methods based on data characteristics and objectives. \texttt{Mean.IMP} and  \texttt{PbP} serve as strong baselines for respectively small and large sample sizes, due to their low training time, and their good predictive performances. Improved performance can be obtained for various sample sizes using \texttt{MICE.RF.Y}. 

\section*{Acknowledgments}
This work is part of the DIGPHAT project which was
supported by a grant from the French government,
managed by the National Research Agency (ANR),
under the France 2030 program, with reference ANR22-PESN-0017.

\bibliography{ref}
\bibliographystyle{icml2026}

\newpage
\appendix
\onecolumn
\section{Proofs}
\label{app_proofs}

\begin{lemma}
\label{lem:probit_integral}
For any $a,b \in \mathds{R}$, letting $f_{\mu, \sigma^2}$ the density of a Gaussian $\mathcal{N}(\mu, \sigma^2)$, we have
\begin{align}
    \int \Phi(t + ax) f_{\mu, \sigma^2}(x) \textrm{d}x & = \Phi\left( \frac{t + a \mu}{\sqrt{1 + a^2 \sigma^2}}\right). 
\end{align}
\end{lemma}

\begin{proof}[Proof of \Cref{lem:probit_integral}]

\begin{align*}
\int \Phi(t + ax) f_{\mu, \sigma^2}(x) \textrm{d}x &= \frac{1}{\sigma \sqrt{2\pi}}\int  \Phi(t + ax) \exp\left(-\frac{(x- \mu)^2}{2\sigma^2}\right) \, dx \\
& = \frac{1}{ \sqrt{2\pi}}\int  \Phi(t + a(\sigma v + \mu)) \exp\left(-\frac{v^2}{2}\right) \, dv \\
& = \frac{1}{ \sqrt{2\pi}}\int  \Phi(b + a\sigma v ) \exp\left(-\frac{v^2}{2}\right) \, dv,
\end{align*}
with $b=t + a \mu$.

\begin{align*}
\Phi(  b + a \sigma v ) & = \int_{-\infty}^{ b + a \sigma v} \frac{1}{\sqrt{2\pi}} \exp\left(-\frac{y^2}{2}\right)\, dy \\
& = \int_{-\infty}^{ b} \frac{1}{\sqrt{2\pi}} \exp\left(-\frac{(u + a \sigma v)^2}{2}\right)\, du.
\end{align*}
Thus, 
\begin{align*}
\int \Phi(t + ax) f_{\mu, \sigma^2}(x) \textrm{d}x &= \frac{1}{ \sqrt{2\pi}}\int  \int_{-\infty}^{ b} \frac{1}{\sqrt{2\pi}} \exp\left(-\frac{(u + a \sigma v)^2}{2}\right)\, du \exp\left(-\frac{v^2}{2}\right) \, dv \\
&= \frac{1}{ \sqrt{2\pi}}\int  \int_{-\infty}^{ b} \frac{1}{\sqrt{2\pi}} \exp\left(-\frac{(u + a \sigma v)^2}{2} -\frac{v^2}{2}\right)\, du   dv.
\end{align*}
Noticing that 
\begin{align*}
-\frac{(u + a \sigma v)^2}{2} -\frac{v^2}{2} & = - \frac{1}{2} (1 + a^2\sigma^2) \left(v + \frac{ua\sigma}{1 + a^2 \sigma^2}\right)^2 - \frac{u^2 }{2(1 + a^2 \sigma^2)},   
\end{align*}
we obtain
\begin{align*}
\int \Phi(t + ax) f_{\mu, \sigma^2}(x) \textrm{d}x & = \frac{1}{ \sqrt{2\pi}} \frac{1}{\sqrt{1+a^2 \sigma^2}}  \int_{-\infty}^{ b}  \exp\left(- \frac{u^2 }{2(1 + a^2 \sigma^2)}\right)\, du  \\
& = \frac{1}{ \sqrt{2\pi}}   \int_{-\infty}^{ b/\sqrt{1 + a^2 \sigma^2}}  \exp\left(- \frac{v^2 }{2}\right)\, dv  \\
& = \Phi\left( \frac{t + a \mu}{\sqrt{1 + a^2 \sigma^2}}\right).
\end{align*}

\end{proof}

\subsection{Proof of \Cref{th:probit_model}}
\label{app:proofs_1}

Fix a pattern $m \in \{0,1\}^d$. The Bayes predictor on pattern $m$ is defined as 
\begin{align}
\mathds{P}[Y=1 | X_{obs(m)}, M=m] & = \mathds{E}[ \mathds{P}[Y=1 | X, M=m] | X_{obs(m)}, M=m] \\
& = \mathds{E}[ \mathds{P}[Y=1 | X] | X_{obs(m)}, M=m] \\
& = \mathds{E}[ \Phi ( \beta_0^\star + \sum_{j=1}^d \beta_j^\star X_j)  | X_{obs(m)}, M=m].
\end{align}
where we have used the fact that $Y$ is independent of $M$ conditional on $X$, due to the GPMM assumption. 
Thus, for any $x \in \mathds{R}^{|obs(m)|}$, 
\begin{align}
\mathds{P}[Y=1 | X_{obs(m)} = x, M=m]  
& = \mathds{E}[ \Phi ( u + Z)  | X_{obs(m)} = x, M=m],
\end{align}
letting $u = \beta_0^\star + \sum_{j \in obs(m)} \beta_j^\star x_j$ and $Z = \sum_{j \notin obs(m)} \beta_j^\star X_j$. By assumption, $X | M = m$ is Gaussian. Thus, the distribution of $X_{mis(m)}$ conditional on $X_{obs(m)} = x, M=m$ is also Gaussian, distributed as 
$\mathcal{N}(\mu_m'(x), \Sigma_m')$ with \citep[see, e.g.,][]{majumdar2019conditional}  
\begin{align}
\mu'_m(x) & = \mu_{m,mis(m)} + \Sigma_{m,mis(m), obs(m)} \Sigma_{m,obs(m), obs(m)}^{-1} (x - \mu_{m,obs(m)}) \\
\Sigma'_m & = \Sigma_{m,mis(m), mis(m)} - \Sigma_{m,mis(m), obs(m)} \Sigma_{m,obs(m), obs(m)}^{-1}  \Sigma_{m,obs(m), mis(m)}.
\end{align}
Since $ Z = (\beta^{\star}_{mis(m)})^\top X_{mis(m)}$, the distribution of $Z$ conditional on $X_{obs(m)} = x$ and $M=m$ is $\mathcal{N}(\tilde{\mu}_m(x) , \tilde{\sigma}^2_m )$, with
\begin{align}
    \tilde{\mu}_m(x) & = (\beta^{\star}_{mis(m)})^\top\mu'_m(x)\\
    \tilde{\sigma}^2_m & = (\beta^{\star}_{mis(m)})^\top \Sigma'_m \beta^{\star}_{mis(m)}.
\end{align}
Let $f (z ; \tilde{\mu}_m(x), \tilde{\sigma}^2_m)$ be the density of a univariate Gaussian with parameters $(\tilde{\mu}_m(x), \tilde{\sigma}^2_m)$. Then, the Bayes predictor on pattern $m$ takes the form 
\begin{align}
\mathds{P}[Y=1 | X_{obs(m)}=x, M=m]  
& = \int \Phi ( u + z )  f (z ; \tilde{\mu}_m(x), \tilde{\sigma}^2_m) \textrm{d}z\\
 & = \Phi\left( \frac{ \tilde{\mu}_m(x) + u  }{\sqrt{1 + \tilde{\sigma}^2_m}}\right),
\end{align}
using \Cref{lem:probit_integral}, where
\begin{align*}
u & = \beta_0^\star + \sum_{j \in obs(m)} \beta_j^\star x_j \\
\tilde{\mu}_m(x)  & = (\beta^{\star}_{mis(m)})^\top\mu_{m,mis(m)}   + (\beta^{\star}_{mis(m)})^\top \Sigma_{m,mis(m), obs(m)} \Sigma_{m,obs(m), obs(m)}^{-1} (x - \mu_{m,obs(m)}).
\end{align*}
Letting 
\begin{align*}
    \alpha_{0,m} & = \beta_0^\star + (\beta^{\star}_{mis(m)})^\top\mu_{m,mis(m)}  - (\beta^{\star}_{mis(m)})^\top \Sigma_{m,mis(m), obs(m)} \Sigma_{m,obs(m), obs(m)}^{-1}  \mu_{m,obs(m)}\\
    \alpha_m & = \beta_{obs(m)}^\star +  \Sigma_{m,obs(m), obs(m)}^{-1} \Sigma_{m,obs(m), mis(m)} \beta^{\star}_{mis(m)},
\end{align*}
we have
\begin{align}
     \mathds{P}[Y=1 | X_{obs(m)} = x, M=m]  = \Phi\left( \frac{ \alpha_{0,m}  +  \alpha_m^\top x  }{\sqrt{1+  \tilde{\sigma}^2_m}}\right),
\end{align}
with $\tilde{\sigma}^2_m  = (\beta^{\star}_{mis(m)})^\top \Sigma'_m \beta^{\star}_{mis(m)}$ where
\begin{align}
    \Sigma'_m & = \Sigma_{m,mis(m), mis(m)} - \Sigma_{m,mis(m), obs(m)} \Sigma_{m,obs(m), obs(m)}^{-1}  \Sigma_{m,obs(m), mis(m)}.
\end{align}

\subsection{Proof of \Cref{th:approx_logistic}}

\label{app:proofs_2}

Fix a pattern $m \in \{0,1\}^d$. The Bayes predictor on pattern $m$ is defined as 
\begin{align}
\mathds{P}[Y=1 | X_{obs(m)}, M=m] & = \mathds{E}[ \mathds{P}[Y=1 | X, M=m] | X_{obs(m)}, M=m] \\
& = \mathds{E}[ \mathds{P}[Y=1 | X] | X_{obs(m)}, M=m] \\
& = \mathds{E}[ \sigma ( \beta_0^\star + \sum_{j=1}^d \beta_j^\star X_j)  | X_{obs(m)}, M=m],
\end{align}
where we have used the fact that $Y$ is independent of $M$ conditional on $X$, due to the GPMM assumption. 
Thus, for any $x \in \mathds{R}^{|obs(m)|}$, 
\begin{align}
\mathds{P}[Y=1 | X_{obs(m)} = x, M=m]  
& = \mathds{E}[ \sigma ( u + Z)  | X_{obs(m)} = x, M=m],
\end{align}
letting $u = \beta_0^\star + \sum_{j \in obs(m)} \beta_j^\star x_j$ and $Z = \sum_{j \notin obs(m)} \beta_j^\star X_j$. By assumption, $X|M=m$ is Gaussian. Thus, the distribution of $X_{mis(m)}$ conditional on $X_{obs(m)} = x$, $M=m$ is also Gaussian, distributed as 
$\mathcal{N}(\mu'_m(x), \Sigma_m')$ with \citep[see, e.g.,][]{majumdar2019conditional}  
\begin{align}
\mu'_m(x) &= \mu_{m,mis(m)} + \Sigma_{m,mis(m), obs(m)} \Sigma_{m,obs(m), obs(m)}^{-1} (x - \mu_{m,obs(m)}) \\
\Sigma'_m & = \Sigma_{m,mis(m), mis(m)} - \Sigma_{m,mis(m), obs(m)} \Sigma_{m,obs(m), obs(m)}^{-1}  \Sigma_{m,obs(m), mis(m)}.
\end{align}
Since $ Z = (\beta^{\star}_{mis(m)})^\top X_{mis(m)}$, the distribution of $Z$ conditional on $X_{obs(m)} = x$ and $M = m$ is $\mathcal{N}(\tilde{\mu}_m(x) , \tilde{\sigma}^2_m )$, with
\begin{align}
    \tilde{\mu}_m(x) & = (\beta^{\star}_{mis(m)})^\top\mu'_m(x)\\
    \tilde{\sigma}^2_m & = (\beta^{\star}_{mis(m)})^\top \Sigma'_m \beta^{\star}_{mis(m)}.
\end{align}

Let $f (z ; \tilde{\mu}_m(x), \tilde{\sigma}^2_m)$ be the density of a univariate Gaussian with parameters $(\tilde{\mu}_m(x), \tilde{\sigma}^2_m)$. Then, the Bayes predictor on pattern $m$ takes the form 
\begin{align}
\mathds{P}[Y=1 | X_{obs(m)}=x, M=m]  
& = \int \sigma ( u + z )  f (z ; \tilde{\mu}_m(x), \tilde{\sigma}^2_m) \textrm{d}z.
\end{align}
Recall that the probit function is defined as $$\Phi(t) = \int_{-\infty}^t \frac{1}{\sqrt{2\pi}} e^{-t^2/2} \textrm{d}t.$$
Letting $\varepsilon(t) = \Phi(t) - \sigma(\sqrt{8/\pi} t)$, we have numerically that $\|\varepsilon\|_{\infty} \simeq 0.018$. Thus, we use the following decomposition:
\begin{align}
 \mathds{P}[Y=1 | X_{obs(m)}=x, M=m] &  = \int \Phi(\sqrt{\pi/8} ( u + z ))  f (z ; \tilde{\mu}_m(x), \tilde{\sigma}^2_m) \textrm{d}z \\
& \qquad + \int (\sigma (u+z) - \Phi(\sqrt{\pi/8} ( u + z )))  f (z ; \tilde{\mu}_m(x), \tilde{\sigma}^2_m) \textrm{d}z.
\end{align}
Using \Cref{lem:probit_integral}, we obtain 
\begin{align}
\int \Phi(\sqrt{\pi/8} ( u + z ))  f (z ; \tilde{\mu}_m(x), \tilde{\sigma}^2_m) \textrm{d}z = \Phi\left( \frac{ \tilde{\mu}_m(x) + u  }{\sqrt{(8/\pi) + \tilde{\sigma}^2_m}}\right).
\end{align}
Thus, 
\begin{align}
 & \mathds{P}[Y=1 | X_{obs(m)}=x, M=m] \\ 
&  =  \Phi\left( \frac{ \tilde{\mu}_m(x) + u  }{\sqrt{(8/\pi) + \tilde{\sigma}^2_m}}\right)  + \int (\sigma (u+z) - \Phi(\sqrt{\pi/8} ( u + z )))  f (z ; \tilde{\mu}_m(x), \tilde{\sigma}^2_m) \textrm{d}z \\
& = \sigma\left( \frac{ \tilde{\mu}_m(x) + u  }{\sqrt{1+ (\pi/8) \tilde{\sigma}^2_m}}\right) + \left( \Phi\left( \frac{ \tilde{\mu}_m(x) + u  }{\sqrt{(8/\pi) + \tilde{\sigma}^2_m(x)}}\right) - \sigma\left( \frac{ \tilde{\mu}_m(x) + u  }{\sqrt{1+ (\pi/8) \tilde{\sigma}^2_m}}\right) \right) \\
& \qquad + \int (\sigma (u+z) - \Phi(\sqrt{\pi/8} ( u + z )))  f (z ; \tilde{\mu}_m(x), \tilde{\sigma}^2_m) \textrm{d}z.
\end{align}
Finally, we have
\begin{align}
\left| \mathds{P}[Y=1 | X_{obs(m)} = x, M=m]  - \sigma\left( \frac{ \tilde{\mu}_m(x) + u  }{\sqrt{1+ (\pi/8) \tilde{\sigma}^2_m}}\right) \right| \leq 2 \|\varepsilon\|_{\infty},
\end{align}

with 
\begin{align*}
u & = \beta_0^\star + \sum_{j \in obs(m)} \beta_j^\star x_j \\
\tilde{\mu}_m(x)  & = (\beta^{\star}_{mis(m)})^\top\mu_{m,mis(m)} \\
&\qquad + (\beta^{\star}_{mis(m)})^\top \Sigma_{m,mis(m), obs(m)} \Sigma_{m,obs(m), obs(m)}^{-1} (x - \mu_{m,obs(m)}).
\end{align*}
Letting 
\begin{align*}
    \alpha_{0,m} & = \beta_0^\star + (\beta^{\star}_{mis(m)})^\top\mu_{m,mis(m)} \\
    & \qquad - (\beta^{\star}_{mis(m)})^\top \Sigma_{m, mis(m), obs(m)} \Sigma_{m, obs(m), obs(m)}^{-1}  \mu_{m, obs(m)}\\
    \alpha_m & = (\beta_{obs(m)}^\star)^\top x + (\beta^{\star}_{mis(m)})^\top \Sigma_{m, mis(m), obs(m)} \Sigma_{m, obs(m), obs(m)}^{-1} x,
\end{align*}
we have
\begin{align}
    \left| \mathds{P}[Y=1 | X_{obs(m)} = x, M=m]  - \sigma\left( \frac{ \alpha_{0,m}  +  \alpha_m^\top x  }{\sqrt{1+ (\pi/8) \tilde{\sigma}^2_m}}\right) \right| \leq 2 \|\varepsilon\|_{\infty},
\end{align}
with $\tilde{\sigma}^2_m  = (\beta^{\star}_{mis(m)})^\top \Sigma'_m \beta^{\star}_{mis(m)}$ where
\begin{align}
    \Sigma'_m & = \Sigma_{m,mis(m), mis(m)} - \Sigma_{m,mis(m), obs(m)} \Sigma_{m,obs(m), obs(m)}^{-1}  \Sigma_{m,obs(m), mis(m)}.
\end{align}

\section{Details on Experimental Setting}

\subsection{Estimating Bayes Probabilities}
\label{app:howtocomputepstar}

Assuming that the complete data follows a logistic model independent from the missingness mask, i.e. $\P[Y|X,M]=\P[Y|X]=\sigma((\beta^\star)^\top X)$, we have
\begin{align}
   \eta^\star_m(x_{obs(m)}) & = \mathds{E}[Y|X_{obs(M)} = x_{obs}, M = m] \\
   & = \mathds{E}[ \mathds{E}[Y | X, M=m] |X_{obs(M)} = x_{obs}, M = m] \\
   & = \mathds{E}[ \mathds{E}[Y | X] |X_{obs(M)} = x_{obs}, M = m] \\
   & = \mathds{E}[ \sigma((\beta^\star)^\top X) |X_{obs(M)} = x_{obs}, M = m].
\end{align}

For any pattern $m \in \{0,1\}^d$, and any $x_{obs(m)}$, we compute the Bayes probability $\eta^\star_m(x_{obs(m)})$ as follows: 
\begin{enumerate}
    \item We sample $k$ observations $x_{mis(m)}^1, \hdots, x_{mis(m)}^k$ from the conditional distribution of $X_{mis(m)} \mid X_{obs(m)} = x_{obs(m)}$. In our simulations (MCAR settings with Gaussian inputs, GPMM-MAR or GPMM-MNAR), this distribution is Gaussian with known parameters (mean and covariance matrix). We let $x^{(1)}, \ldots, x^{(k)}$ be the full observations, obtained by combining $x_{obs(m)}$ and each generated vector $x_{mis(m)}^\ell$.
    
    \item We estimate the Bayes probability $\eta^\star_m(x_{obs(m)})$ by the Monte Carlo average $\frac{1}{k}\sum_{i=1}^k \etaStar(x^{(i)})$, where $\etaStar(x) = \sigma ((\beta^\star)^\top x)$.
\end{enumerate}

In the model of \Cref{sec:additional_simulations_nonlinearfeatures}, we compute the Bayes probabilities as follows, leveraging the latent Gaussian distribution of the latent features. Recall that in this model, $X$ is Gaussian and $Z = g(X) = (g_j(X_j))_{j=1, \hdots, 5}$ where $g$ is invertible. In the MCAR setting,  
\begin{align}
   \eta^\star_m(z_{obs(m)}) & = \mathds{E}[Y|Z_{obs(M)} = z_{obs}, M = m] \\
   & = \mathds{E}[ \mathds{E}[Y | Z, M=m] |Z_{obs(M)} = z_{obs}, M = m] \\
   & = \mathds{E}[ \mathds{E}[Y | Z] |Z_{obs(M)} = z_{obs}, M = m] \\
   & = \mathds{E}[ \sigma((\beta^\star)^\top Z) |Z_{obs(M)} = z_{obs}, M = m].
\end{align}
In order to compute the Bayes probability $\eta^\star_m(z_{obs(m)})$: 
\begin{enumerate}
    \item we compute $x_{obs(m)} = g^{-1}_{obs(m)}(z_{obs(m)})$
    \item Since $X_{mis(m)}| X_{obs(m)} = x_{obs(m)}$ is Gaussian with known parameters, one can generate complete observations $x^1, \hdots, x^k$, with observed components $x_{obs(m)}$ and missing components generated as $X_{mis(m)}| X_{obs(m)} = x_{obs(m)}$.
    \item Let $z^1, \hdots, z^k$ such that $z^\ell = g(x^{\ell})$.
    \item We estimate the Bayes probability $\eta^\star_m(z_{obs(m)})$ by the Monte Carlo average $\frac{1}{k}\sum_{\ell=1}^k \etaStar(z^{(\ell)})$, where $\etaStar(z) = \sigma ((\beta^\star)^\top z)$.
\end{enumerate}

\subsection{Calibration measure} \label{sec:risk_measures_appendix}

\citet{dimitriadis2021stable} proposes a consistent approach to recalibrate probability estimations using isotonic regression (via the Pool-Adjacent-Violators algorithm). A decomposition of the Brier score is constructed from this recalibration. The Miscalibration (MCB) component is defined as the difference between the mean Brier score of the original forecast probabilities, $\eta$, and the mean Brier score of the suitably recalibrated probabilities, $\eta^c$:
\begin{equation} \text{MCB} = \frac{1}{n} \sum_{i=1}^n (\eta_i - Y_i)^2 - \frac{1}{n} \sum_{i=1}^n (\eta^c_i - Y_i)^2. \end{equation} Here, $\eta^c$ represents the isotonic estimator of the true event probability. Consequently, MCB quantifies the reduction in Brier score achieved by perfect recalibration, representing the forecast's deviation from perfect calibration. This score is implemented via the \textsc{R} package \texttt{reliabilitydiag}.

\section{Simulations}
\subsection{Gaussian features (MCAR)}

\begin{table}[h!]
\centering
\caption{Average training and prediction time, in seconds, of the procedures for different training sample sizes, for the experiment described in \ref{sec:methodo_SimA}. All simulations were ran on an internal cluster, using CPUs only.}
\begin{tabular}{l c c c c c c c}
\toprule
\textbf{Algorithms} & \multicolumn{6}{c}{\textbf{Training}} & \multicolumn{1}{c}{\textbf{Prediction}} \\
\cmidrule(lr){2-7} \cmidrule(lr){8-8}
& 100 & 500 & 1000 & 5000 & 10000 & 50000 & 15000 \\
\midrule
\texttt{05.IMP} & 0.006 & 0.007 & 0.006 & 0.015 & 0.060 & 0.231 & 0.013 \\
\texttt{05.IMP.M} & 0.007 & 0.006 & 0.007 & 0.020 & 0.140 & 0.343 & 0.024 \\
\texttt{Mean.IMP} & 0.008 & 0.005 & 0.045 & 0.050 & 0.095 & 0.383 & 0.015 \\
\texttt{Mean.IMP.M} & 0.006 & 0.006 & 0.008 & 0.100 & 0.079 & 0.383 & 0.013 \\
\texttt{MICE.1.IMP} & 0.372 & 0.409 & 0.384 & 0.464 & 0.737 & 1.971 & 0.012 \\
\texttt{MICE.1.IMP.M} & 0.464 & 0.324 & 0.415 & 0.495 & 0.787 & 2.191 & 0.027 \\
\texttt{MICE.1.M.IMP} & 0.569 & 0.513 & 0.627 & 0.740 & 1.123 & 2.996 & 0.020 \\
\texttt{MICE.1.M.IMP.M} & 0.525 & 0.547 & 0.659 & 0.809 & 0.965 & 2.873 & 0.028 \\
\texttt{MICE.1.Y.IMP} & 0.340 & 0.460 & 0.480 & 0.766 & 1.055 & 2.982 & 0.020 \\
\texttt{MICE.1.Y.IMP.M} & 0.366 & 0.492 & 0.487 & 0.747 & 1.052 & 2.561 & 0.027 \\
\texttt{MICE.1.Y.M.IMP} & 0.526 & 0.676 & 0.706 & 0.946 & 1.198 & 3.205 & 0.016 \\
\texttt{MICE.1.Y.M.IMP.M} & 0.432 & 0.639 & 0.658 & 0.948 & 1.105 & 3.334 & 0.022 \\
\texttt{MICE.10.IMP} & 2.244 & 2.549 & 2.830 & 3.789 & 5.145 & 15.878 & 0.113 \\
\texttt{MICE.10.IMP.M} & 1.781 & 1.892 & 1.901 & 3.880 & 5.242 & 15.839 & 0.195 \\
\texttt{MICE.10.M.IMP} & 2.762 & 3.287 & 3.943 & 5.412 & 7.222 & 20.752 & 0.142 \\
\texttt{MICE.10.M.IMP.M} & 3.586 & 3.473 & 3.648 & 5.332 & 7.487 & 21.886 & 0.251 \\
\texttt{MICE.10.Y.IMP} & 2.610 & 2.802 & 2.923 & 5.026 & 6.745 & 20.090 & 0.152 \\
\texttt{MICE.10.Y.IMP.M} & 2.511 & 3.032 & 3.538 & 5.007 & 6.692 & 19.519 & 0.210 \\
\texttt{MICE.10.Y.M.IMP} & 4.688 & 4.582 & 4.596 & 6.876 & 9.543 & 27.645 & 0.180 \\
\texttt{MICE.10.Y.M.IMP.M} & 3.632 & 3.661 & 3.859 & 7.464 & 9.229 & 26.415 & 0.244 \\
\texttt{MICE.100.IMP} & 10.521 & 11.527 & 12.836 & 18.963 & 27.920 & 88.253 & 0.633 \\
\texttt{MICE.100.IMP.M} & 9.740 & 10.597 & 12.398 & 17.844 & 26.094 & 90.693 & 0.894 \\
\texttt{MICE.100.M.IMP} & 16.441 & 17.864 & 19.500 & 27.720 & 38.745 & 122.452 & 0.688 \\
\texttt{MICE.100.M.IMP.M} & 15.552 & 16.995 & 17.669 & 28.129 & 39.959 & 130.258 & 1.060 \\
\texttt{MICE.100.Y.IMP} & 14.202 & 15.399 & 16.761 & 25.502 & 36.757 & 113.436 & 0.607 \\
\texttt{MICE.100.Y.IMP.M} & 13.783 & 14.794 & 16.431 & 24.389 & 34.885 & 118.433 & 1.006 \\
\texttt{MICE.100.Y.M.IMP} & 21.270 & 23.161 & 24.485 & 37.048 & 48.413 & 158.226 & 0.741 \\
\texttt{MICE.100.Y.M.IMP.M} & 19.806 & 22.349 & 23.881 & 35.385 & 50.326 & 169.439 & 1.110 \\
\texttt{MICE.RF.10.IMP} & 43.400 & 49.307 & 51.109 & 94.532 & 131.992 & 559.480 & 0.133 \\
\texttt{MICE.RF.10.IMP.M} & 43.537 & 51.034 & 53.388 & 94.540 & 132.099 & 575.752 & 0.227 \\
\texttt{MICE.RF.10.M.IMP} & 44.343 & 52.237 & 53.382 & 92.365 & 127.217 & 490.945 & 0.161 \\
\texttt{MICE.RF.10.M.IMP.M} & 44.630 & 52.275 & 53.069 & 90.778 & 121.516 & 508.571 & 0.266 \\
\texttt{MICE.RF.10.Y.IMP} & 50.231 & 72.512 & 83.388 & 121.470 & 177.875 & 613.060 & 0.134 \\
\texttt{MICE.RF.10.Y.IMP.M} & 51.630 & 74.429 & 82.628 & 122.824 & 179.444 & 606.207 & 0.253 \\
\texttt{MICE.RF.10.Y.M.IMP} & 53.064 & 75.109 & 83.732 & 122.853 & 169.484 & 578.275 & 0.151 \\
\texttt{MICE.RF.10.Y.M.IMP.M} & 55.766 & 78.454 & 84.918 & 123.331 & 178.718 & 590.501 & 0.273 \\
\texttt{CC} & 0.006 & 0.009 & 0.007 & 0.015 & 0.026 & 0.187 & --- \\
\texttt{PbP} & 0.068 & 0.186 & 0.203 & 0.776 & 1.559 & 5.669 & 0.270 \\
\texttt{SAEM} & 3.761 & 6.188 & 9.156 & 32.877 & 60.517 & 275.664 & 11.738 \\
\bottomrule
\end{tabular}
\label{tab:runtimeSimA}
\end{table}

\begin{figure*}[h!]
\centering
\includegraphics[width=0.85\textwidth]{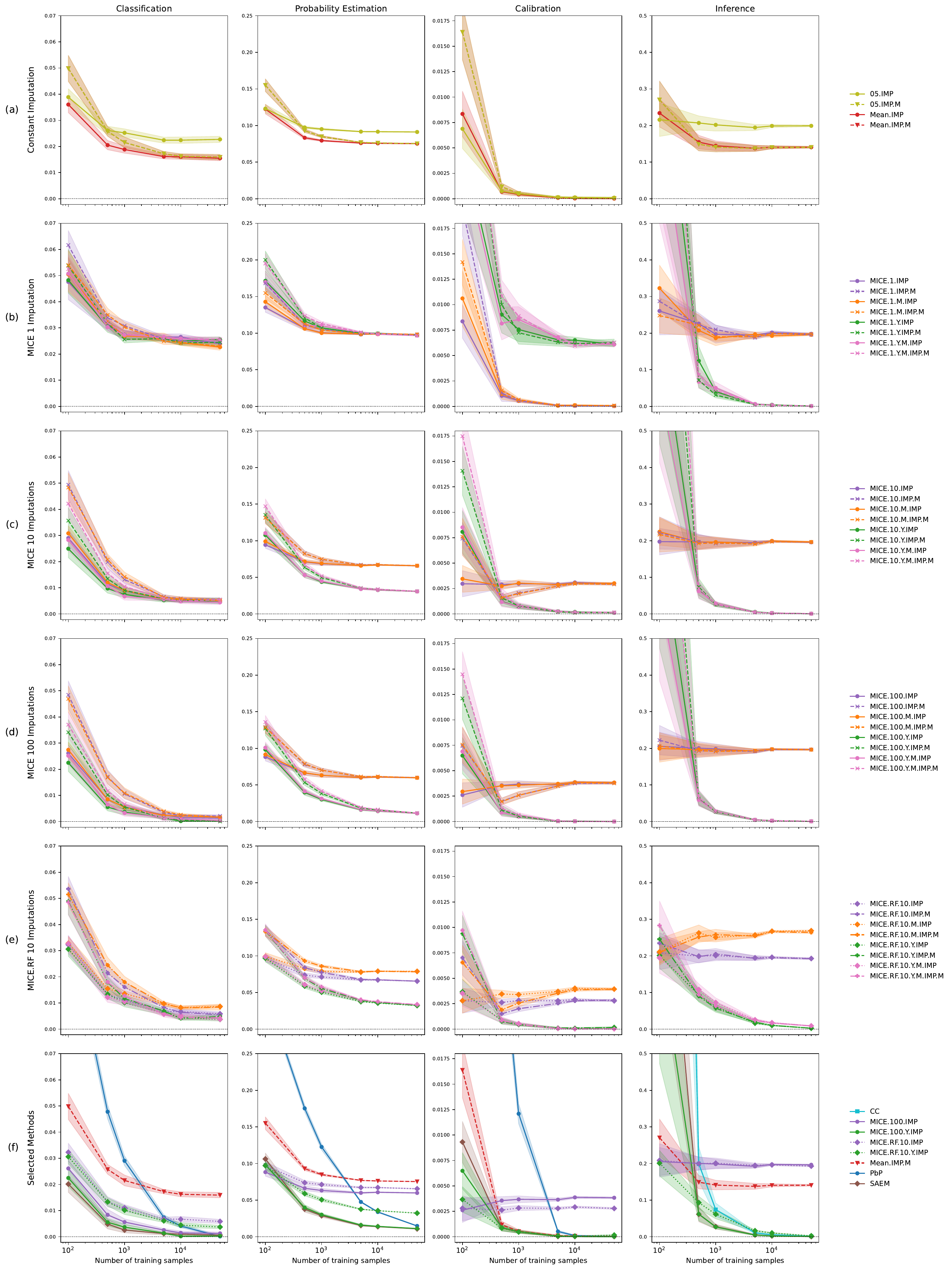}
\vspace{-0.2cm}
\caption{Aggregate results of the simulation GPMM (MCAR). A panel of methods are evaluated on Classification (via misclassification rate), Probability Estimation (via MAE from Bayes probabilities), Calibration (via CORP-MCB) and Inference (via MSE of $\widehat{\beta}$). Mean and standard errors over 10 replicates are displayed. Note that the curves from \texttt{Mean.IMP.M} and \texttt{05.IMP.M} overlap, as do those from \texttt{MICE.10.M.IMP} and \texttt{MICE.10.IMP} for some metrics.}
\label{fig:SimA_ALL}
\end{figure*}

\clearpage
\subsection{Gaussian Pattern Mixture Model (MAR)}
\label{app:gpmm_mar}

\subsubsection{Methodology}

In this GPMM-MAR scenario, the feature dimension is set to $d=5$, as in the other settings. The first two features are always observed.  For all patterns $m$, we decompose $\mu_m$ into $\mu_{m,1:2}$ and $\mu_{m,3:5}$ with 
$\mu_{m,3:5} = \textbf{0}$ and $\mu_{m,1:2}\sim \mathcal{N}(0,0.5I_2)$. Similarly, we let $\Sigma_m$ be a block-diagonal matrix where $\Sigma_{m,3:5\times3:5} = [\rho^{|i-j|}]_{i,j=3}^5$ with $\rho = 0.65$ and  $\Sigma_{m,1:2\times1:2}$ as $\sigma_m [\rho_m^{|i-j|}]_{i,j=1}^2$, where, for each $m$, we sample $\rho_{m}\sim \mathcal{U}([-1,1])$, $\sigma_m\sim \mathcal{U}([0,1])$. The components of the missing mask $M$ are independent Bernoulli random variables $M_j \sim \mathcal{B}(p_j)$ with $p = [0,0,0.25,0.25,0.25]$.

\subsubsection{Results}
\label{sec:results_mar}

\Cref{fig:SimE_ALL} displays the results for the MAR setting. \texttt{MICE.1} has poor performances in classification and probability estimation. \texttt{MICE.100} has the best overall performance; considering RF imputations (\texttt{MICE.RF.10}) does not improve the performances.  Contrary to the MCAR Gaussian feature case, incorporating the labels $Y$ into multiple MICE imputation procedures improves performances across all metrics. This is also true for \texttt{MICE.RF}. Consistent with our theoretical findings, \texttt{PbP} demonstrates excellent performance in all metrics for large sample sizes. In classification and probability estimation, \texttt{SAEM} and \texttt{MICE.100.Y.IMP} (with or without mask in either imputation or logistic model) exhibit the best overall performance alongside \texttt{PbP}. 
Regarding calibration and inference, the results are largely similar to those observed with Gaussian features.

\begin{figure*}[ht!]
\centering
\includegraphics[width=0.7\textwidth]{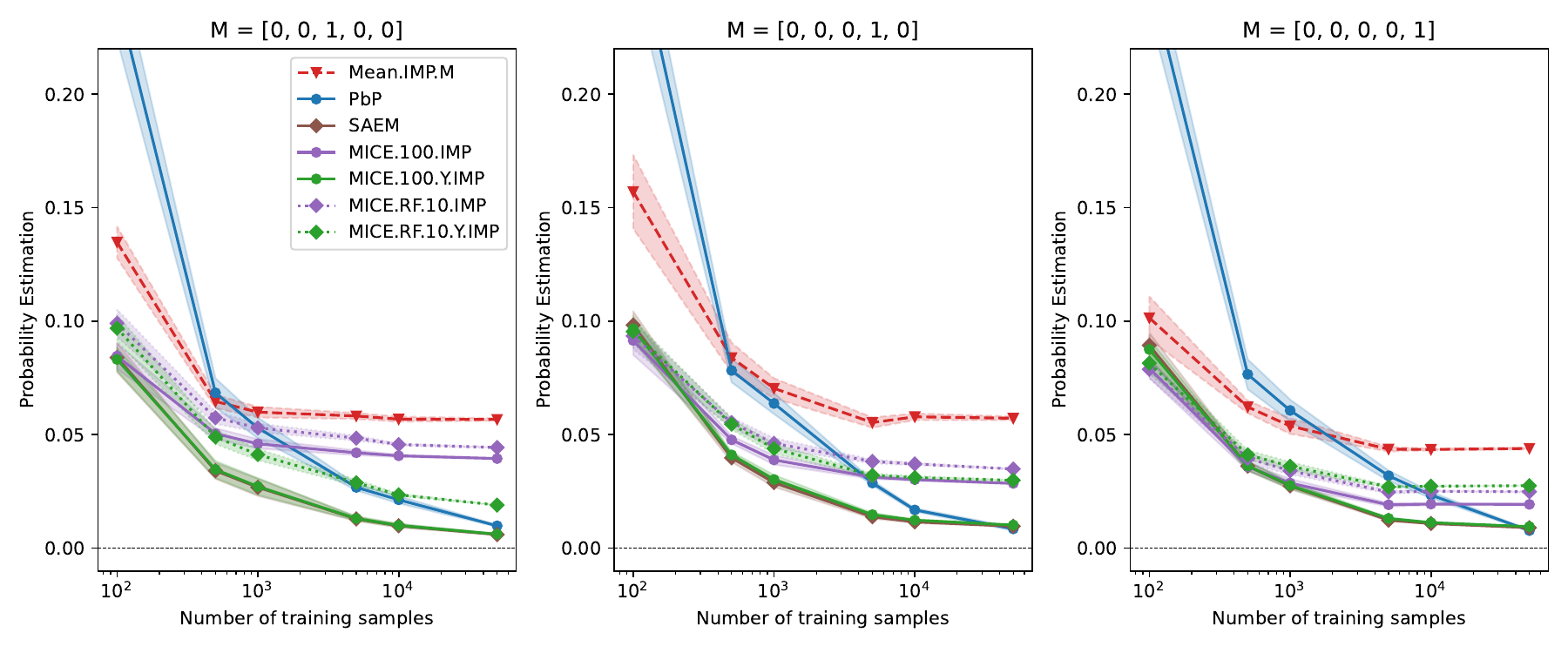}
\vspace{-0.2cm}
\caption{Performances of selected procedures in terms of MAE from Bayes probabilities. The results are displayed by missing pattern in the test set (with one missing index: [0,0,1,0,0], [0,0,0,1,0], [0,0,0,0,1]). Means and standard errors over 10 replicates are displayed. Note that the curves from \texttt{SAEM} and \texttt{MICE.100.Y.IMP} overlap.}
\label{fig:SimE_MAE_patterns}
\end{figure*}

\begin{figure*}[h!]
\centering
\includegraphics[width=0.85\textwidth]{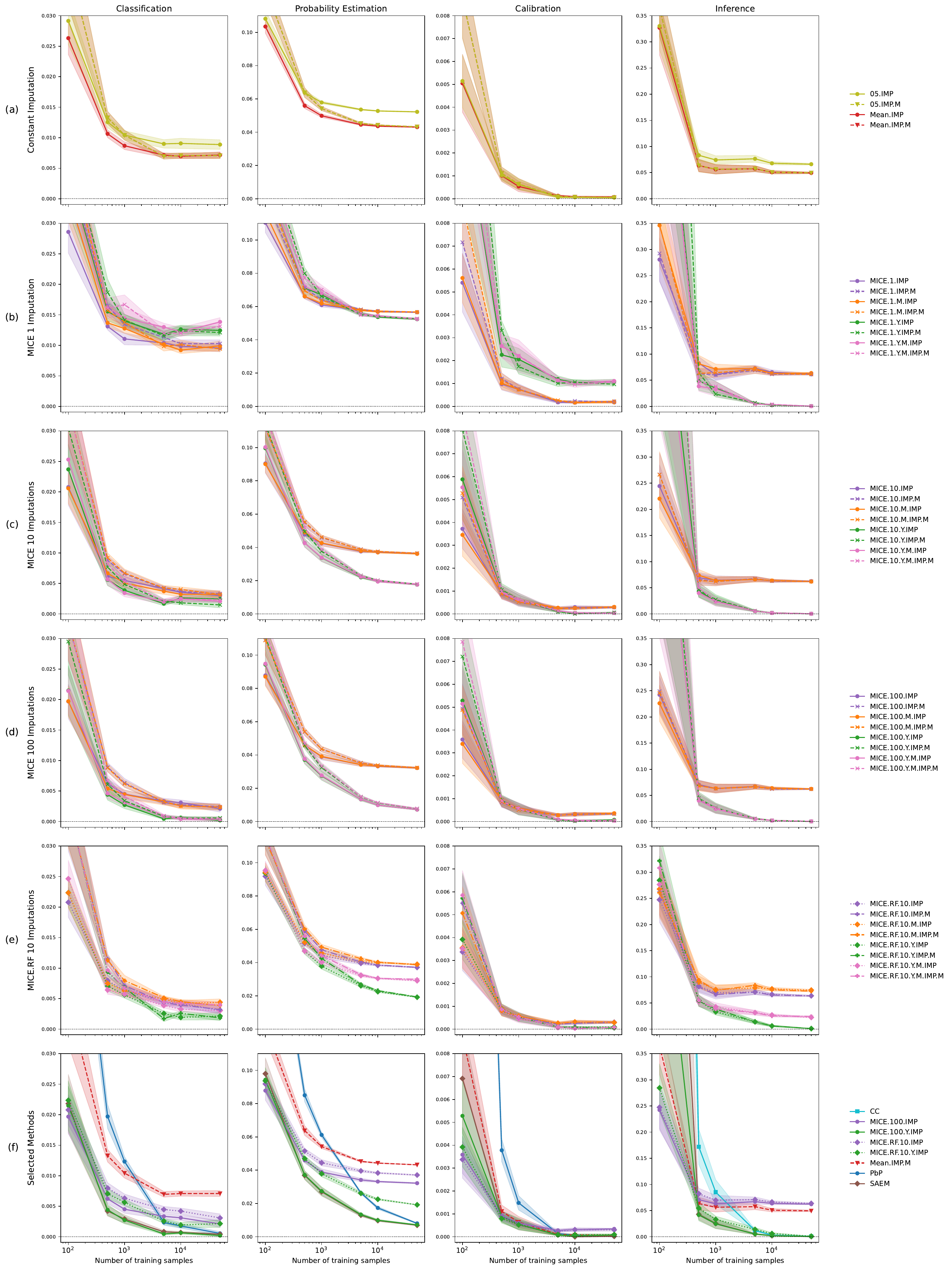}
\vspace{-0.1cm}
\caption{Aggregate results of the simulation GPMM (MAR). A panel of methods are evaluated on Classification (via misclassification rate), Probability Estimation (via MAE from Bayes probabilities), Calibration (via CORP-MCB), and Inference (via MSE of $\widehat{\beta}$). Mean and standard errors over 10 replicates are displayed.}
\label{fig:SimE_ALL}
\end{figure*}

\clearpage
\subsection{Gaussian Pattern Mixture Model (MNAR)}

\begin{figure*}[h!]
\centering
\includegraphics[width=0.85\textwidth]{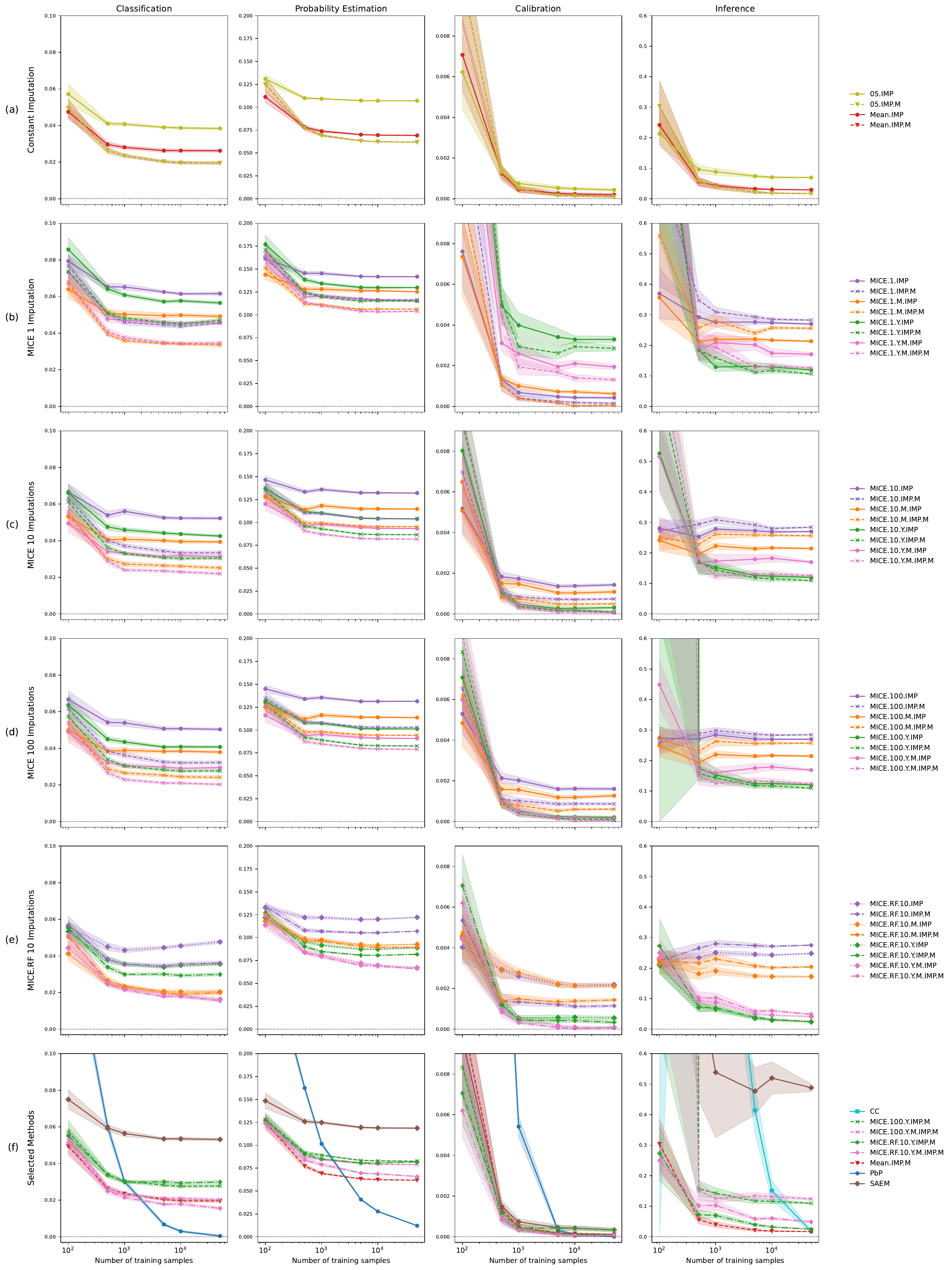}
\vspace{-0.1cm}
\caption{Aggregate results of the simulation GPMM (MNAR). A panel of methods are evaluated on Classification (via misclassification rate), Probability Estimation (via MAE from Bayes probabilities), Calibration (via CORP-MCB), and Inference (via MSE of $\widehat{\beta}$). Mean and standard errors over 10 replicates are displayed. Note that \texttt{SAEM} did not converge in two replicates for a training sample size of $100$; statistics are therefore based on the remaining 8 replicates}
\label{fig:SimG_ALL}
\end{figure*}

\begin{figure*}[ht!]
\centering
\includegraphics[width=\textwidth]{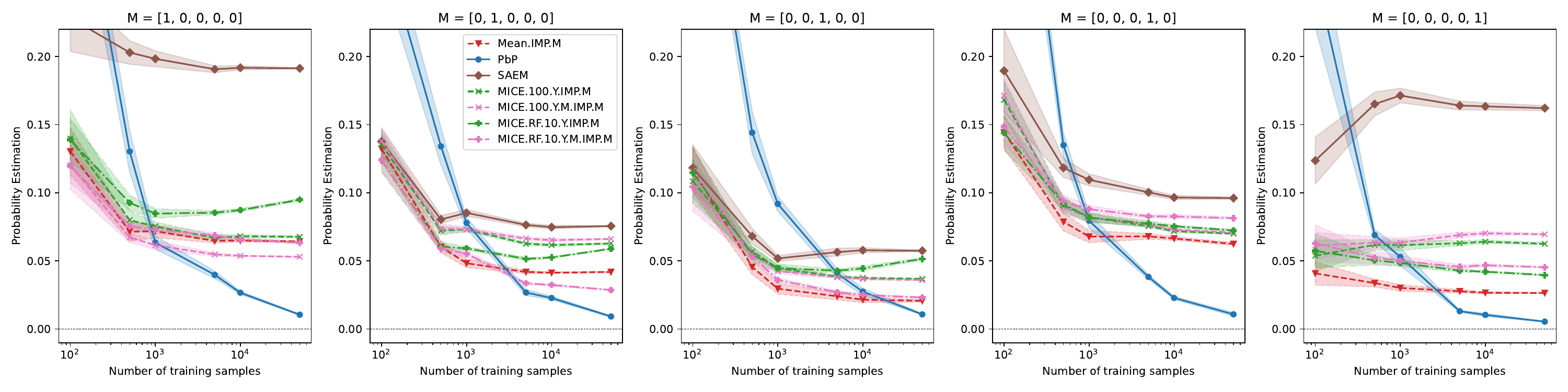}
\vspace{-0.2cm}
\caption{Performances of selected procedures in terms of MAE from Bayes probabilities. The results are displayed by missing pattern in the test set (with one missing index: [1,0,0,0,0], ..., [0,0,0,0,1]). Means and standard errors over 10 replicates are displayed.}
\label{fig:SimG_MAE_patterns}
\end{figure*}

\clearpage
\subsection{Non-linear features (MCAR)}
\label{sec:app_nonlinearfeatures}
The feature vector $\mathbf{Z}$ is defined as 
\begin{itemize}
    \item $Z_1 = \X_1$ and $Z_2 = \X_2$ (identity transformations).
    \item $Z_3 = \exp(\X_3) + c_3$.
    \item $Z_4 = (\X_4)^3$.
    \item $Z_5 = c_5 + \begin{cases} (\X_5)^2 & \text{if } \X_5 \geq 0 \\ -10 \exp(\X_5) & \text{if } \X_5 < 0.\end{cases}$
\end{itemize}
The constants $c_3$ and $c_5$ are chosen to ensure that the mean of each transformed feature remains approximately zero, contributing to a balanced outcome distribution ($\P(Y=1) \approx 0.5$): $c_3 = (-1.67)$ and $c_4 = 2$. The distribution of these features relative to $Z_1$ is visualized in \Cref{fig:SimC_Xi_vs_X0}. 

\begin{figure*}[h]
\centering
\includegraphics[width=\textwidth]{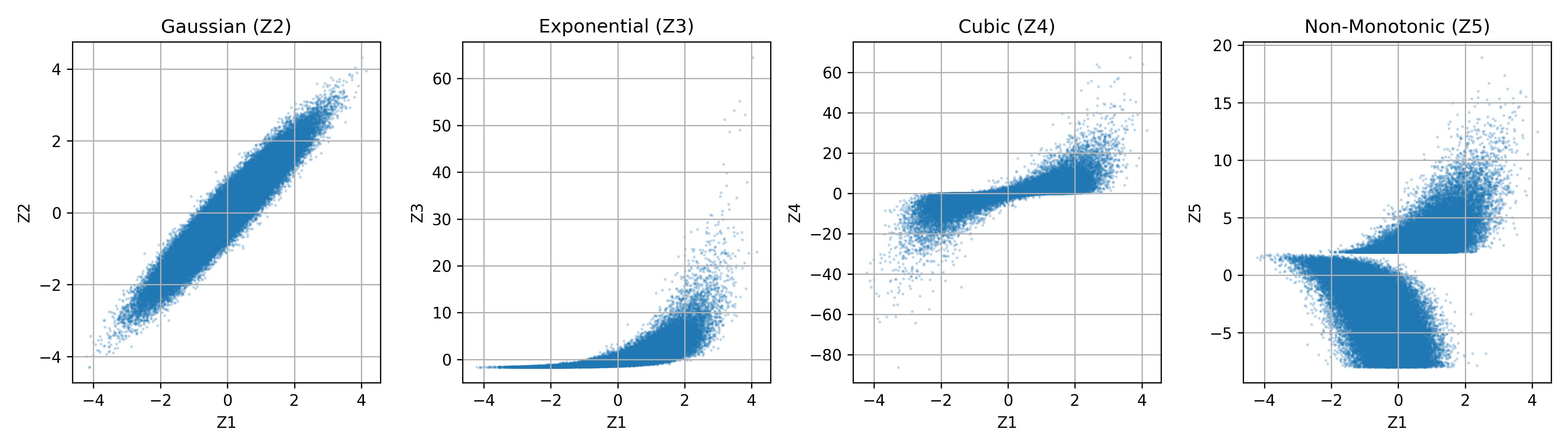}
\vspace{-0.2cm}
\caption{Visualization of the transformed features $Z_2, Z_3, Z_4, Z_5$ against the Gaussian feature $Z_1$ in the non-linear simulation. Each subplot illustrates a different transformation type: Gaussian (identity), Exponential, Cubic, and Non-Monotonic, for the first replicate of the simulation with non-linear features and MCAR missingness described in \Cref{sec:methodo_SimA}.}
\label{fig:SimC_Xi_vs_X0}
\end{figure*}

\begin{figure*}[ht!]
\centering
\includegraphics[width=\textwidth]{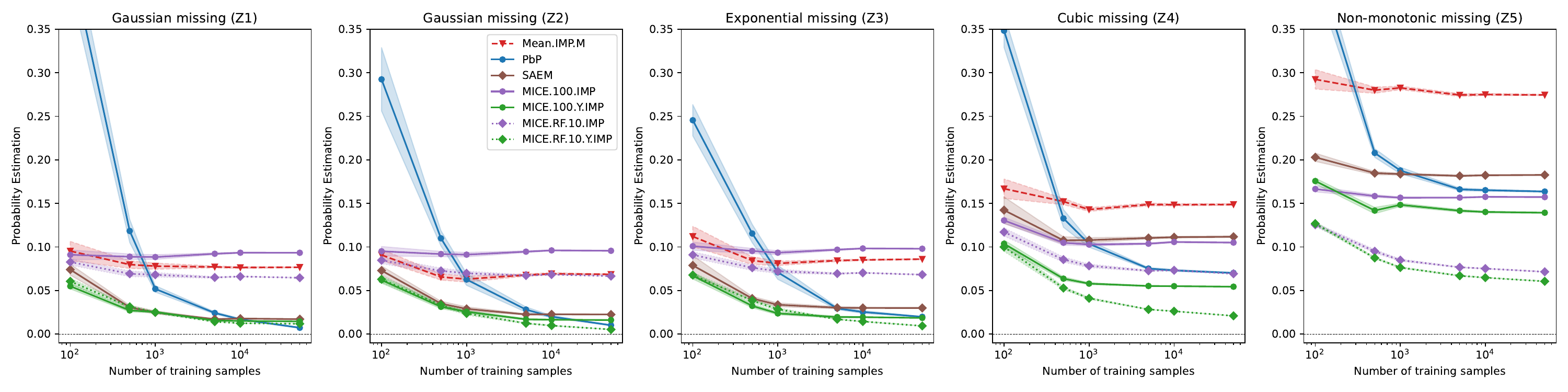}
\vspace{-0.2cm}
\caption{Performances of selected procedures in terms of MAE from Bayes probabilities. The results are displayed by missing pattern in the test set (with one missing index: [1,0,0,0,0], ..., [0,0,0,0,1]). Means and standard errors over 10 replicates of non-linear features with MCAR missingness are displayed (see \Cref{sec:methodo_SimA}).}
\label{fig:SimC_MAE_patterns}
\end{figure*}

\begin{figure*}
\centering
\includegraphics[width=0.85\textwidth]{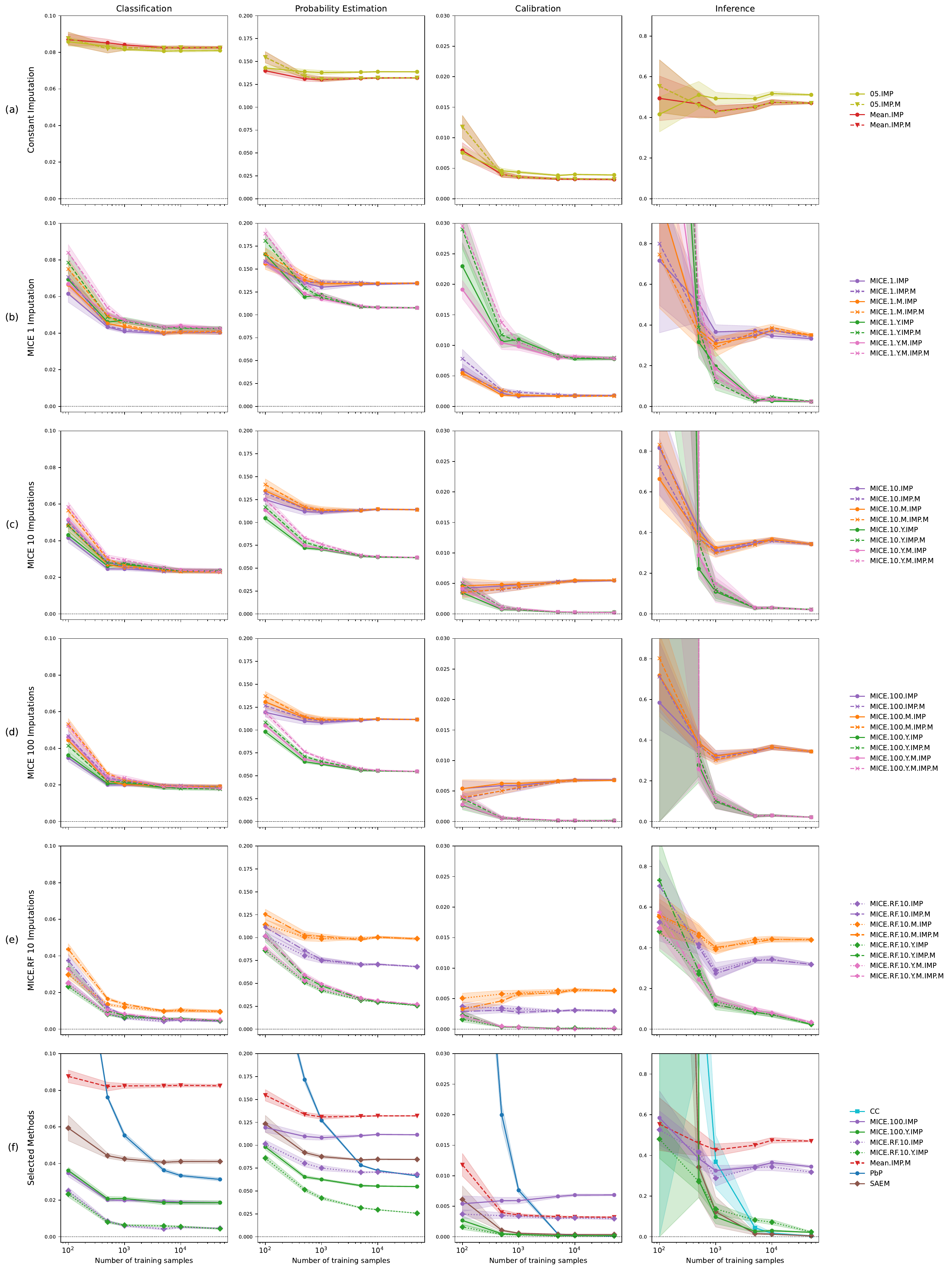}
\vspace{-0.1cm}
\caption{Aggregate results of the simulation Non-linear (MCAR). A panel of methods are evaluated on Classification (via misclassification rate), Probability Estimation (via MAE from Bayes probabilities), Calibration (via CORP-MCB), and Inference (via MSE of $\widehat{\beta}$). Mean and standard errors over 10 replicates are displayed. Note that \texttt{SAEM} did not converge in three replicates for a training sample size of $100$; statistics are therefore based on the remaining 7 replicates.}
\label{fig:SimC_ALL}
\end{figure*}

\clearpage
\subsection{Non-linear features (MNAR)}
\label{app:non_linear_MNAR}

\subsubsection{Methodology}

In this scenario, we combine the non-linear feature transformations from \Cref{sec:app_nonlinearfeatures} with a Missing Not At Random (MNAR) mechanism. 

\paragraph{Data Generation:} We generate the covariates $X$ using the same non-linear transformations of a latent Gaussian vector $Z$ as described in \Cref{sec:app_nonlinearfeatures}. Specifically, $Z \sim \mathcal{N}(0, \Sigma)$ with a Toeplitz covariance structure ($\rho=0.95$), and $X$ is obtained via component-wise transformations (identity, exponential, cubic, and piecewise functions). The outcome $Y$ is generated according to a logistic model given the complete $X$:

$$\mathbb{P}(Y=1|X) = \sigma(\beta_0 + \beta^\top X).$$

\paragraph{Missingness Mechanism:} Unlike the MCAR setting in \Cref{sec:app_nonlinearfeatures}, the missingness mask $M$ is generated using a self-masking mechanism where the probability of missingness depends on the values of $X$ itself. For each observation $i$, the missingness probabilities are determined by a logistic model:

$$\mathbb{P}(M_{ij}=1 | X_i) = \sigma(\gamma_0 + \gamma_1 \tilde{X}_{ij}),$$

where $\tilde{X}$ is the standardized version of $X$ (column-wise). We set the intercept $\gamma_0 = -0.5$ and slope $\gamma_1 = 1.0$. This creates a self-masking MNAR mechanism where extreme values of $X$ are more likely to be missing. This mechanism leads to $37.91\%$ of missing values on average (with s.d. $<0.05$ over $10$ replicates).

\paragraph{Bayes Probability Estimation:} Since the analytical derivation of the Bayes probabilities $\mathbb{P}(Y=1 | X_{obs}, M)$ is intractable under this complex non-linear MNAR scheme, we approximate them numerically. We generate a large auxiliary dataset ($N=100,000$) following the same distribution. We then train a Random Forest classifier (200 estimators, max depth 15) on this auxiliary data to predict $Y$ given $(X_{obs}, M)$ (where missing values in $X$ are handled by the Random Forest's internal mechanism). This model is treated as the ground truth Bayes predictor for evaluating the performance of the methods on the test set.

\subsubsection{Results}
\Cref{fig:SimNLMNAR_ALL} displays the performance of a panel of methods in this non-gaussian MNAR setting. 

\paragraph{Superiority of Non-linear imputation}
Consistent with the GPMM-MNAR simulation (\Cref{fig:SimG_ALL}), non-linear imputation methods (MICE with Random Forest) are the best performing method. This is expected given the non-linear nature of the covariates.

\paragraph{Impact of the Mask and Labels}
As for Non-linear MCAR setting (\Cref{fig:SimC_ALL}), the incorporation of the labels $Y$ and of the missingness mask \emph{before} the imputation improves considerably the performance of MICE techniques.

\paragraph{SAEM \& PbP}
\texttt{SAEM} suffers from the non-linearity of the covariates and shows poor performance across all metrics of interest, as expected. On the other hand, \texttt{PbP} shows competitive results for high sample size, approaching the results of the more advanced MICE methods.

\begin{figure*}
\centering
\includegraphics[width=0.85\textwidth]{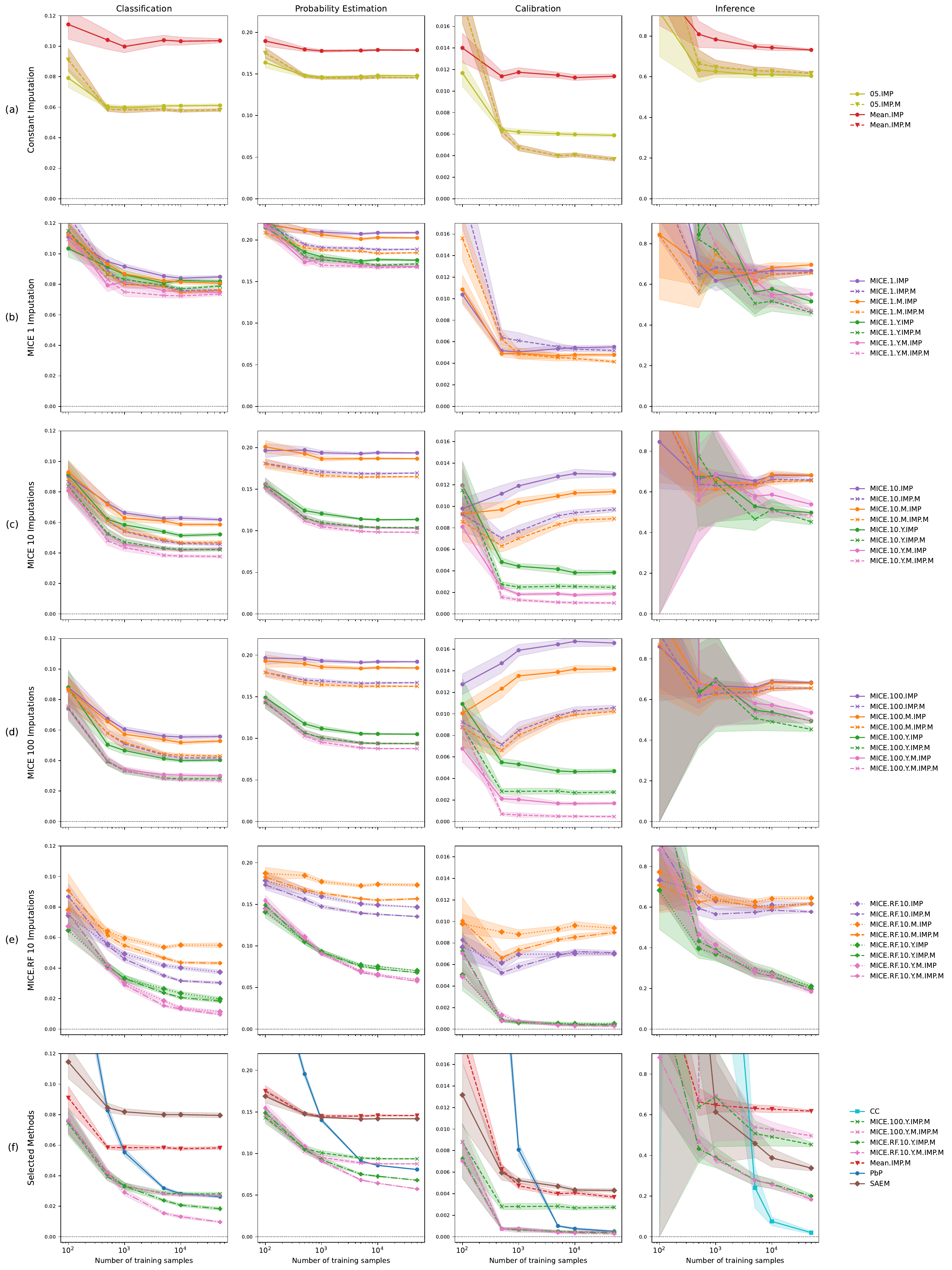}
\vspace{-0.1cm}
\caption{Aggregate results of the simulation Non-linear (MNAR). A panel of methods are evaluated on Classification (via misclassification rate), Probability Estimation (via MAE from Bayes probabilities), Calibration (via CORP-MCB), and Inference (via MSE of $\widehat{\beta}$). Mean and standard errors over 10 replicates are displayed. Note that \texttt{SAEM} did not converge in five replicates for a training sample size of $100$; statistics are therefore based on the remaining 5 replicates.}
\label{fig:SimNLMNAR_ALL}
\end{figure*}

\clearpage
\section{Summary Tables}

\begin{table}[h!]
\centering
\scriptsize
\caption{Average misclassification rates and training times for selected methods across five simulations. Scores within one standard deviation of the best (minimum) mean are in \textbf{bold}, and those within two standard deviations are \underline{underlined}.}
\begin{tabular}{l ccccc c ccccc c}
\toprule
\multirow{2}{*}{Methods}
& \multicolumn{6}{c}{Low sample size ($N=100$)} & \multicolumn{6}{c}{High sample size ($N=50,000$)} \\
\cmidrule(lr){2-7} \cmidrule(lr){8-13}
& \shortstack{GPMM\\MCAR} & \shortstack{GPMM\\MAR} & \shortstack{GPMM\\MNAR} & \shortstack{NL\\MCAR} & \shortstack{NL\\MNAR} & Time & \shortstack{GPMM\\MCAR} & \shortstack{GPMM\\MAR} & \shortstack{GPMM\\MNAR} & \shortstack{NL\\MCAR} & \shortstack{NL\\MNAR} & Time \\
\midrule
{\tiny \texttt{PbP}} & 0.12 & 0.07 & 0.17 & 0.21 & 0.22 & 0.07 & \textbf{0.000} & \textbf{0.001} & \textbf{0.000} & 0.031 & 0.026 & 4.5 \\
{\tiny \texttt{SAEM}} & \textbf{0.02} & \textbf{0.02} & 0.07 & 0.06 & 0.11 & 3.01 & \textbf{0.001} & \textbf{0.000} & 0.053 & 0.041 & 0.080 & 234.3 \\
{\tiny \texttt{Mean.IMP}} & \underline{0.04} & \textbf{0.03} & \textbf{0.05} & 0.09 & 0.11 & \textbf{0.01} & 0.015 & 0.007 & 0.026 & 0.083 & 0.104 & \textbf{0.3} \\
{\tiny \texttt{Mean.IMP.M}} & 0.05 & 0.04 & \textbf{0.05} & 0.09 & \underline{0.09} & \textbf{0.02} & 0.016 & 0.007 & 0.020 & 0.082 & 0.058 & 0.4 \\
{\tiny \texttt{MICE.1.IMP}} & 0.05 & \underline{0.03} & 0.08 & 0.06 & 0.11 & 0.33 & 0.024 & 0.010 & 0.062 & 0.040 & 0.085 & 1.9 \\
{\tiny \texttt{MICE.1.Y.IMP}} & 0.05 & 0.04 & 0.09 & 0.07 & 0.10 & 0.34 & 0.026 & 0.012 & 0.056 & 0.043 & 0.082 & 2.6 \\
{\tiny \texttt{MICE.1.Y.M.IMP.M}} & 0.05 & 0.04 & \underline{0.07} & 0.08 & 0.11 & 0.43 & 0.026 & 0.013 & 0.035 & 0.043 & 0.074 & 3.3 \\
{\tiny \texttt{MICE.100.IMP}} & \textbf{0.03} & \textbf{0.02} & \underline{0.07} & 0.03 & \underline{0.09} & 15.54 & \textbf{0.001} & 0.002 & 0.050 & 0.019 & 0.056 & 112.8 \\
{\tiny \texttt{MICE.100.Y.IMP}} & \textbf{0.02} & \textbf{0.02} & \underline{0.06} & 0.04 & \underline{0.09} & 20.55 & \textbf{0.000} & \textbf{0.000} & 0.041 & 0.019 & 0.040 & 146.9 \\
{\tiny \texttt{MICE.100.Y.M.IMP.M}} & \underline{0.04} & \underline{0.03} & \textbf{0.05} & 0.05 & \textbf{0.08} & 27.99 & \textbf{0.000} & \textbf{0.000} & 0.020 & 0.018 & 0.027 & 201.8 \\
{\tiny \texttt{MICE.RF.10.IMP}} & \underline{0.03} & \textbf{0.02} & \textbf{0.06} & \textbf{0.03} & \textbf{0.07} & 30.58 & 0.006 & 0.003 & 0.048 & \textbf{0.005} & 0.037 & 458.7 \\
{\tiny \texttt{MICE.RF.10.Y.IMP}} & \underline{0.03} & \textbf{0.02} & \textbf{0.05} & \textbf{0.02} & \textbf{0.06} & 46.58 & 0.004 & 0.002 & 0.036 & \textbf{0.004} & 0.020 & 504.8 \\
{\tiny \texttt{MICE.RF.10.Y.M.IMP.M}} & 0.05 & \underline{0.04} & \textbf{0.05} & \underline{0.03} & \textbf{0.08} & 48.51 & 0.004 & 0.004 & 0.016 & \textbf{0.005} & \textbf{0.010} & 471.6 \\
\bottomrule
\end{tabular}
\label{tab:misclassification_summary}
\end{table}

\begin{table}[h!]
\centering
\scriptsize
\caption{Average probability estimation MAEs and training times for selected methods across five simulations. Scores within one standard deviation of the best (minimum) mean are in \textbf{bold}, and those within two standard deviations are \underline{underlined}.}
\begin{tabular}{l ccccc c ccccc c}
\toprule
\multirow{2}{*}{Methods}
& \multicolumn{6}{c}{Low sample size ($N=100$)} & \multicolumn{6}{c}{High sample size ($N=50,000$)} \\
\cmidrule(lr){2-7} \cmidrule(lr){8-13}
& \shortstack{GPMM\\MCAR} & \shortstack{GPMM\\MAR} & \shortstack{GPMM\\MNAR} & \shortstack{NL\\MCAR} & \shortstack{NL\\MNAR} & Time & \shortstack{GPMM\\MCAR} & \shortstack{GPMM\\MAR} & \shortstack{GPMM\\MNAR} & \shortstack{NL\\MCAR} & \shortstack{NL\\MNAR} & Time \\
\midrule
{\tiny \texttt{PbP}} & 0.33 & 0.23 & 0.3 & 0.34 & 0.4 & 0.07 & 0.01 & \underline{0.008} & \textbf{0.01} & 0.07 & 0.08 & 4.5 \\
{\tiny \texttt{SAEM}} & \underline{0.11} & \textbf{0.10} & \underline{0.1} & 0.12 & \underline{0.2} & 3.01 & \textbf{0.01} & \textbf{0.007} & 0.12 & 0.08 & 0.14 & 234.3 \\
{\tiny \texttt{Mean.IMP}} & \underline{0.12} & \textbf{0.10} & \textbf{0.1} & 0.14 & 0.2 & \textbf{0.01} & 0.08 & 0.043 & 0.07 & 0.13 & 0.18 & \textbf{0.3} \\
{\tiny \texttt{Mean.IMP.M}} & 0.16 & 0.13 & \textbf{0.1} & 0.15 & \underline{0.2} & \textbf{0.02} & 0.08 & 0.043 & 0.06 & 0.13 & 0.15 & 0.4 \\
{\tiny \texttt{MICE.1.IMP}} & 0.14 & \underline{0.11} & 0.2 & 0.16 & 0.2 & 0.33 & 0.10 & 0.057 & 0.14 & 0.13 & 0.21 & 1.9 \\
{\tiny \texttt{MICE.1.Y.IMP}} & 0.17 & 0.14 & 0.2 & 0.17 & 0.2 & 0.34 & 0.10 & 0.052 & 0.13 & 0.11 & 0.18 & 2.6 \\
{\tiny \texttt{MICE.1.Y.M.IMP.M}} & 0.20 & 0.15 & 0.2 & 0.19 & 0.2 & 0.43 & 0.10 & 0.053 & 0.10 & 0.11 & 0.17 & 3.3 \\
{\tiny \texttt{MICE.100.IMP}} & \textbf{0.09} & \textbf{0.09} & \underline{0.1} & 0.12 & 0.2 & 15.54 & 0.06 & 0.032 & 0.13 & 0.11 & 0.19 & 112.8 \\
{\tiny \texttt{MICE.100.Y.IMP}} & \textbf{0.10} & \textbf{0.09} & \underline{0.1} & \underline{0.10} & \textbf{0.1} & 20.55 & \underline{0.01} & \textbf{0.007} & 0.10 & 0.05 & 0.10 & 146.9 \\
{\tiny \texttt{MICE.100.Y.M.IMP.M}} & 0.14 & \underline{0.11} & \textbf{0.1} & 0.12 & \textbf{0.1} & 27.99 & 0.01 & \underline{0.008} & 0.08 & 0.05 & 0.09 & 201.8 \\
{\tiny \texttt{MICE.RF.10.IMP}} & \textbf{0.10} & \textbf{0.09} & \underline{0.1} & \underline{0.10} & \underline{0.2} & 30.58 & 0.07 & 0.037 & 0.12 & 0.07 & 0.15 & 458.7 \\
{\tiny \texttt{MICE.RF.10.Y.IMP}} & \textbf{0.10} & \textbf{0.09} & \textbf{0.1} & \textbf{0.09} & \textbf{0.1} & 46.58 & 0.03 & 0.019 & 0.09 & \textbf{0.03} & 0.07 & 504.8 \\
{\tiny \texttt{MICE.RF.10.Y.M.IMP.M}} & 0.14 & \underline{0.12} & \textbf{0.1} & \underline{0.10} & \textbf{0.2} & 48.51 & 0.03 & 0.030 & 0.07 & 0.03 & \textbf{0.06} & 471.6 \\
\bottomrule
\end{tabular}
\label{tab:mae_bayes_summary}
\end{table}

\begin{table}[h!]
\centering
\scriptsize
\caption{Average calibration scores and training times for selected methods across five simulations. Scores within one standard deviation of the best (minimum) mean are in \textbf{bold}, and those within two standard deviations are \underline{underlined}.}
\begin{tabular}{l ccccc c ccccc c}
\toprule
\multirow{2}{*}{Methods}
& \multicolumn{6}{c}{Low sample size ($N=100$)} & \multicolumn{6}{c}{High sample size ($N=50,000$)} \\
\cmidrule(lr){2-7} \cmidrule(lr){8-13}
& \shortstack{GPMM\\MCAR} & \shortstack{GPMM\\MAR} & \shortstack{GPMM\\MNAR} & \shortstack{NL\\MCAR} & \shortstack{NL\\MNAR} & Time & \shortstack{GPMM\\MCAR} & \shortstack{GPMM\\MAR} & \shortstack{GPMM\\MNAR} & \shortstack{NL\\MCAR} & \shortstack{NL\\MNAR} & Time \\
\midrule
{\tiny \texttt{PbP}} & 0.120 & 0.052 & 0.100 & 0.101 & 0.126 & 0.07 & \textbf{0.000} & \textbf{0.000} & \textbf{-0.000} & \textbf{0.000} & \underline{0.001} & 4.5 \\
{\tiny \texttt{SAEM}} & \underline{0.009} & \underline{0.007} & 0.010 & 0.006 & \underline{0.013} & 3.01 & \textbf{0.000} & \textbf{0.000} & 0.000 & \underline{0.000} & 0.004 & 234.3 \\
{\tiny \texttt{Mean.IMP}} & \underline{0.008} & \textbf{0.005} & \underline{0.007} & 0.008 & \underline{0.014} & \textbf{0.01} & \textbf{0.000} & \textbf{0.000} & 0.000 & 0.003 & 0.011 & \textbf{0.3} \\
{\tiny \texttt{Mean.IMP.M}} & 0.016 & 0.009 & 0.011 & 0.012 & 0.018 & \textbf{0.02} & \textbf{0.000} & \textbf{0.000} & \underline{0.000} & 0.003 & 0.004 & 0.4 \\
{\tiny \texttt{MICE.1.IMP}} & \underline{0.008} & \textbf{0.005} & \underline{0.008} & 0.006 & \underline{0.010} & 0.33 & \textbf{0.000} & 0.000 & 0.000 & 0.002 & 0.006 & 1.9 \\
{\tiny \texttt{MICE.1.Y.IMP}} & 0.026 & 0.013 & 0.022 & 0.023 & 0.036 & 0.34 & 0.006 & 0.001 & 0.003 & 0.008 & 0.020 & 2.6 \\
{\tiny \texttt{MICE.1.Y.M.IMP.M}} & 0.036 & 0.016 & 0.023 & 0.030 & 0.041 & 0.43 & 0.006 & 0.001 & 0.001 & 0.008 & 0.019 & 3.3 \\
{\tiny \texttt{MICE.100.IMP}} & \textbf{0.003} & \textbf{0.004} & \textbf{0.005} & 0.005 & \underline{0.013} & 15.54 & 0.004 & 0.000 & 0.002 & 0.007 & 0.017 & 112.8 \\
{\tiny \texttt{MICE.100.Y.IMP}} & \underline{0.006} & \textbf{0.005} & \underline{0.007} & \textbf{0.003} & \underline{0.011} & 20.55 & \textbf{-0.000} & \textbf{0.000} & 0.000 & \textbf{0.000} & 0.005 & 146.9 \\
{\tiny \texttt{MICE.100.Y.M.IMP.M}} & 0.014 & \underline{0.008} & 0.009 & \underline{0.004} & \textbf{0.009} & 27.99 & \textbf{-0.000} & \textbf{0.000} & \textbf{0.000} & \textbf{0.000} & \underline{0.000} & 201.8 \\
{\tiny \texttt{MICE.RF.10.IMP}} & \textbf{0.003} & \textbf{0.003} & \textbf{0.004} & \underline{0.004} & \textbf{0.008} & 30.58 & 0.003 & 0.000 & 0.002 & 0.003 & 0.007 & 458.7 \\
{\tiny \texttt{MICE.RF.10.Y.IMP}} & \textbf{0.004} & \textbf{0.004} & \textbf{0.005} & \textbf{0.002} & \textbf{0.005} & 46.58 & 0.000 & \underline{0.000} & 0.001 & \textbf{0.000} & \underline{0.001} & 504.8 \\
{\tiny \texttt{MICE.RF.10.Y.M.IMP.M}} & \underline{0.010} & \textbf{0.006} & \textbf{0.006} & \textbf{0.002} & \textbf{0.007} & 48.51 & \textbf{0.000} & \textbf{0.000} & \textbf{0.000} & \textbf{0.000} & \textbf{0.000} & 471.6 \\
\bottomrule
\end{tabular}
\label{tab:calibration_summary}
\end{table}

\begin{table}[h!]
\centering
\scriptsize
\caption{Average MSEs from $\beta^\star$ and training times for selected methods across five simulations. Scores within one standard deviation of the best (minimum) mean are in \textbf{bold}, and those within two standard deviations are \underline{underlined}.}
\begin{tabular}{l ccccc c ccccc c}
\toprule
\multirow{2}{*}{Methods}
& \multicolumn{6}{c}{Low sample size ($N=100$)} & \multicolumn{6}{c}{High sample size ($N=50,000$)} \\
\cmidrule(lr){2-7} \cmidrule(lr){8-13}
& \shortstack{GPMM\\MCAR} & \shortstack{GPMM\\MAR} & \shortstack{GPMM\\MNAR} & \shortstack{NL\\MCAR} & \shortstack{NL\\MNAR} & Time & \shortstack{GPMM\\MCAR} & \shortstack{GPMM\\MAR} & \shortstack{GPMM\\MNAR} & \shortstack{NL\\MCAR} & \shortstack{NL\\MNAR} & Time \\
\midrule
{\tiny \texttt{CC}} & 4.62 & 3.50 & +1000 & +1000 & +1000 & \textbf{0.007} & 0.001 & \underline{0.001} & \textbf{0.02} & \textbf{0.004} & \textbf{0.02} & \textbf{0.2} \\
{\tiny \texttt{SAEM}} & 1.24 & 1.35 & 2.25 & 7.27 & 4.33 & 3.012 & \textbf{0.001} & \textbf{0.000} & 0.49 & \textbf{0.004} & 0.34 & 234.3 \\
{\tiny \texttt{Mean.IMP}} & \textbf{0.23} & \textbf{0.33} & \textbf{0.24} & \textbf{0.49} & \textbf{1.05} & \underline{0.010} & 0.141 & 0.049 & 0.03 & 0.470 & 0.73 & 0.3 \\
{\tiny \texttt{Mean.IMP.M}} & \textbf{0.27} & \underline{0.36} & \underline{0.30} & \textbf{0.55} & \underline{1.29} & 0.015 & 0.141 & 0.050 & \textbf{0.02} & 0.470 & 0.62 & 0.4 \\
{\tiny \texttt{MICE.1.IMP}} & \textbf{0.26} & \textbf{0.28} & \underline{0.37} & \textbf{0.72} & \underline{1.48} & 0.333 & 0.198 & 0.061 & 0.27 & 0.333 & 0.67 & 1.9 \\
{\tiny \texttt{MICE.1.Y.IMP}} & 1.02 & 0.80 & 1.14 & 3.38 & 27.40 & 0.343 & \textbf{0.001} & \textbf{0.000} & 0.12 & 0.023 & 0.52 & 2.6 \\
{\tiny \texttt{MICE.1.Y.M.IMP.M}} & 1.31 & 0.76 & 1.20 & 4.05 & 4.21 & 0.432 & \textbf{0.001} & \textbf{0.000} & 0.13 & 0.021 & 0.47 & 3.3 \\
{\tiny \texttt{MICE.100.IMP}} & \textbf{0.21} & \textbf{0.24} & \textbf{0.27} & \textbf{0.58} & \textbf{0.86} & 15.545 & 0.197 & 0.062 & 0.27 & 0.345 & 0.68 & 112.8 \\
{\tiny \texttt{MICE.100.Y.IMP}} & 0.79 & 0.73 & +1000 & +1000 & +1000 & 20.548 & \textbf{0.001} & \textbf{0.000} & 0.12 & 0.021 & 0.49 & 146.9 \\
{\tiny \texttt{MICE.100.Y.M.IMP.M}} & 1.14 & 0.88 & 7.38 & +1000 & +1000 & 27.989 & \textbf{0.001} & \textbf{0.000} & 0.12 & 0.021 & 0.50 & 201.8 \\
{\tiny \texttt{MICE.RF.10.IMP}} & \textbf{0.21} & \textbf{0.25} & \textbf{0.23} & \textbf{0.53} & \textbf{0.73} & 30.579 & 0.193 & 0.064 & 0.25 & 0.318 & 0.62 & 458.7 \\
{\tiny \texttt{MICE.RF.10.Y.IMP}} & \textbf{0.20} & \textbf{0.28} & \textbf{0.21} & \textbf{0.48} & \textbf{0.68} & 46.581 & 0.003 & 0.001 & 0.02 & 0.024 & 0.21 & 504.8 \\
{\tiny \texttt{MICE.RF.10.Y.M.IMP.M}} & \textbf{0.28} & \textbf{0.31} & \textbf{0.25} & \textbf{0.57} & \textbf{0.88} & 48.515 & 0.009 & 0.024 & 0.05 & 0.034 & 0.18 & 471.6 \\
\bottomrule
\end{tabular}
\label{tab:mse_error_summary}
\end{table}

\clearpage
\section{Real Datasets}
\label{app:real_datasets}

\subsection{Set-up}

We compiled a collection of real-world datasets containing missing values from two sources: the imputation benchmark by \citet{grzesiak2025needdozensmethodsreal} (see their Table 3) and the \textit{R-miss-tastic} platform \citep{rmisstastic}.

From this initial pool, we applied specific exclusion criteria to ensure suitability for our analysis. We discarded time series datasets (e.g., \textit{tsAirgap}, \textit{tsHeating}, \textit{tsNH4}, and World Bank data), datasets exhibiting censored missingness (e.g., \textit{chorizonDL}), and those with hierarchical missingness structures (\textit{riskfactors}). For datasets appearing in both sources (\textit{oceanbuoys}), we retained the pre-processed version from \citet{grzesiak2025needdozensmethodsreal}.

To formulate binary classification tasks, we either selected an existing binary target column or dichotomized a continuous variable (assigning class 0 if values were below the median, and class 1 otherwise). The specific target variable for each dataset is listed in \Cref{tab:real_data_sets}.

To assess the predictive performance and stability of the various approaches, we conduct Monte Carlo cross-validation with $K=15$ iterations. For each iteration, we randomly split the data into a training set (80\%) and a test set (20\%).

\begin{table}[ht]
\centering
\caption{Summary of missing data statistics per dataset. \textit{Patterns} denotes the count of unique missingness patterns. \textit{Coverage 3} and \textit{Coverage 10} indicate the cumulative proportion of observations (in \%) captured by the 3 and 10 most frequent patterns, respectively.}
\begin{tabular}{lrrrrrrll}
  \hline
File & Rows & Cols & Patterns & Missing (\%) & Coverage 3 & Coverage 10 & Source & Target \\ 
  \hline
    airquality & 153 & 6 & 4 & 4.79 & 98.69 & 100.00 & R-miss-tastic & wind \\ 
    boys & 748 & 8 & 13 & 26.75 & 93.98 & 99.60 & benchmark & age\\ 
    colic & 300 & 24 & 186 & 22.94 & 13.00 & 25.67 & benchmark & outcome\\ 
    debt & 464 & 13 & 44 & 4.24 & 76.29 & 89.66 & benchmark & prodebt\\ 
    diabetes & 768 & 9 & 11 & 9.43 & 94.27 & 99.87 & benchmark & class\\ 
    globwarm & 1001 & 10 & 2 & 8.55 & 100.00 & 100.00 & benchmark & chesapeake\\ 
    housevotes84 & 435 & 17 & 76 & 5.30 & 69.43 & 82.07 & benchmark & class\\ 
    nhanes & 10000 & 72 & 2721 & 37.71 & 3.81 & 10.47 & R-miss-tastic & age\\ 
    oceanbuoys & 736 & 8 & 6 & 3.01 & 99.46 & 100.00 & benchmark & wind-ns\\ 
    ozone & 366 & 13 & 13 & 4.27 & 92.90 & 99.18 & R-miss-tastic & v13\\ 
    pedestrian & 37700 & 6 & 2 & 1.13 & 100.00 & 100.00 & R-miss-tastic & sensor-id\\ 
    popmis & 2000 & 6 & 2 & 7.07 & 100.00 & 100.00 & benchmark & teachpop\\ 
    pulplignin & 301 & 22 & 15 & 5.32 & 93.69 & 98.34 & benchmark & y-kappa\\ 
    sbs5242 & 262 & 9 & 14 & 1.91 & 88.17 & 98.47 & R-miss-tastic & usb\\ 
    selfreport & 2060 & 10 & 2 & 20.00 & 100.00 & 100.00 & benchmark & sex\\ 
    sleep & 62 & 10 & 8 & 6.13 & 87.10 & 100.00 & R-miss-tastic & danger\\ 
    soybean & 683 & 36 & 9 & 9.50 & 92.53 & 100.00 & benchmark & class\\ 
    tbc & 3951 & 9 & 6 & 11.65 & 91.37 & 100.00 & benchmark & sex\\ 
    vnf & 1232 & 14 & 102 & 9.07 & 75.81 & 86.04 & benchmark & q8-1\\ 
    walking & 890 & 5 & 4 & 13.62 & 99.33 & 100.00 & benchmark & sex\\ 
   \hline
\end{tabular}
\label{tab:real_data_sets}
\end{table}






\subsection{Results}

Regarding numerical stability and robustness, we observed that several methods failed to produce valid predictions on specific datasets. In such instances, to maintain a complete comparison, we replaced the failing method's outputs with a constant prediction (resulting in an AUC of $0.5$ and a miscalibration error of $0.0$ by construction). Several specific cases are noteworthy:

First, \texttt{SAEM} and \texttt{05.IMP} are restricted to numerical inputs; consequently, they were trained exclusively on the continuous features. For datasets composed entirely of categorical variables (\textit{housevotes84}, \textit{soybean}, and \textit{vnf}), these methods were unable to fit a model. Second, \texttt{SAEM} exhibited convergence instability on the \textit{walking} dataset, preventing valid inference. Finally, both \texttt{SAEM} and the various iterations of \texttt{MICE} with label-dependent imputation (\texttt{MICE.10.Y.M.IMP} and \texttt{MICE.10.Y.M.IMP.M}) failed to execute on the \textit{nhanes} dataset, likely due to the high dimensionality and complexity of the missingness patterns in that specific case.

Figures \ref{fig:real_data1} and \ref{fig:real_data2} present the comparative performance of selected methods across the real-world datasets, evaluated using ROC-AUC, Brier score, Misclassification rate, and Calibration error. Note that parameter inference and probability estimation metrics are omitted for these datasets, as the ground-truth coefficients and true posterior probabilities are unknown, and real-world data distributions likely deviate from the theoretical logistic model. A detailed breakdown of the AUC results is provided in \Cref{tab:auc_results}.

We observe that the variants of MICE imputation generally emerge as the most robust methods. While the choice of the underlying conditional model (Predictive Mean Matching vs. Random Forest) and the inclusion of a missingness mask \emph{prior} to imputation appear to have a negligible impact, the use of a mask \emph{following} imputation—but before the final logistic regression—has a significant influence on performance. This effect is predominantly negative, as seen in the \textit{sleep}, \textit{sbs5242}, and \textit{boys} datasets, although it occasionally yields improvements, notably in the \textit{oceanbuoys} dataset.

Simpler baseline methods exhibit inconsistent performance; while they lack the overall reliability of the MICE framework, they occasionally achieve competitive results in specific contexts. For instance, \texttt{PbP} achieves the highest AUC in the \textit{pedestrian}, \textit{oceanbuoys}, and \textit{tbc} datasets, but its performance degrades substantially in other settings, such as \textit{ozone} and \textit{diabetes}.


\begin{figure*}
\centering
\includegraphics[width=0.95\textwidth]{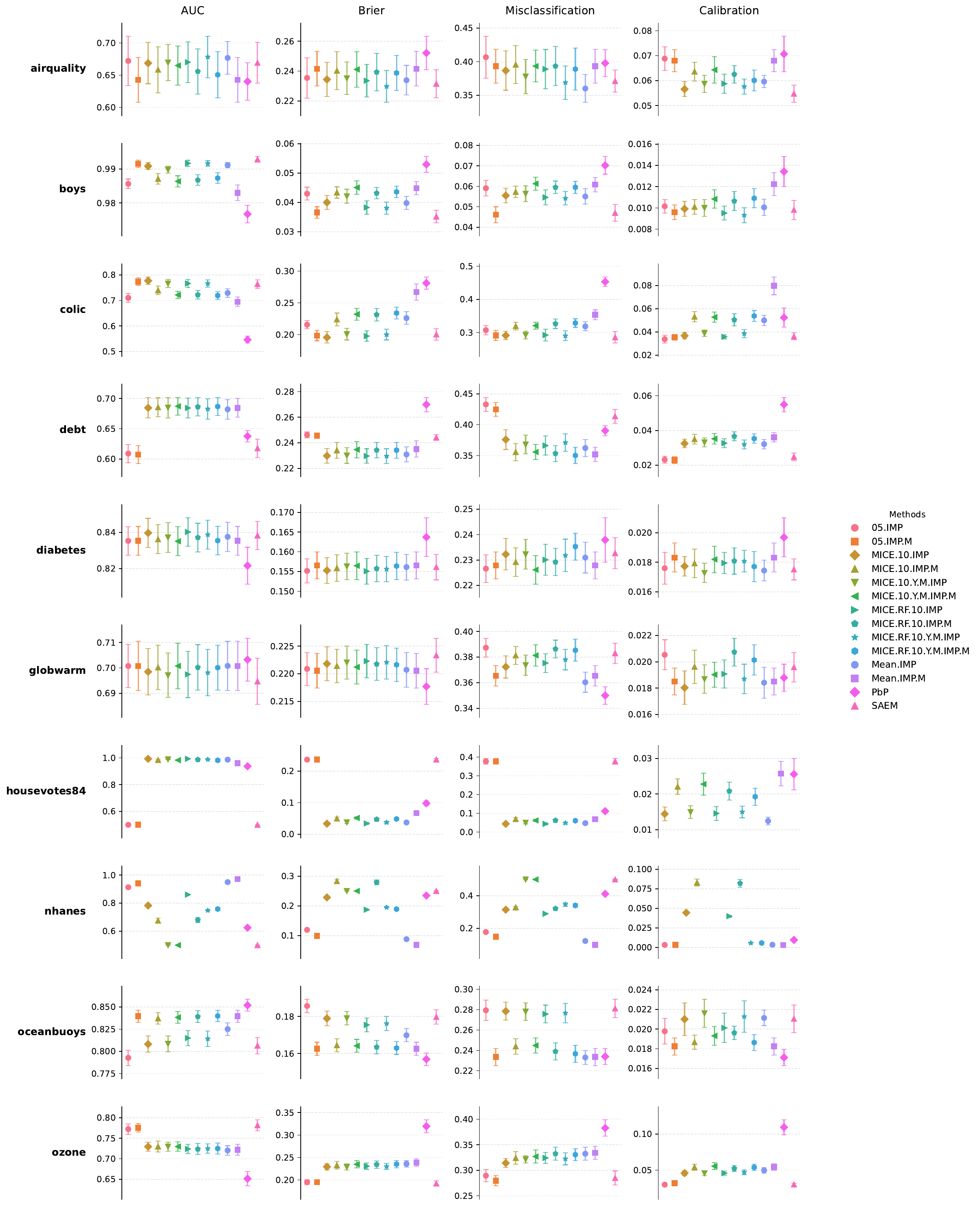}
\vspace{-0.1cm}
\caption{Results on real-world datasets (airquality to ozone). Performances are evaluated on ROC-AUC, Brier score, Misclassification rate, and Calibration error (via CORP-MCB). Points and error bars represent the mean and standard error across $K=15$ Monte Carlo folds.}
\label{fig:real_data1}
\end{figure*}

\begin{figure*}
\centering
\includegraphics[width=0.95\textwidth]{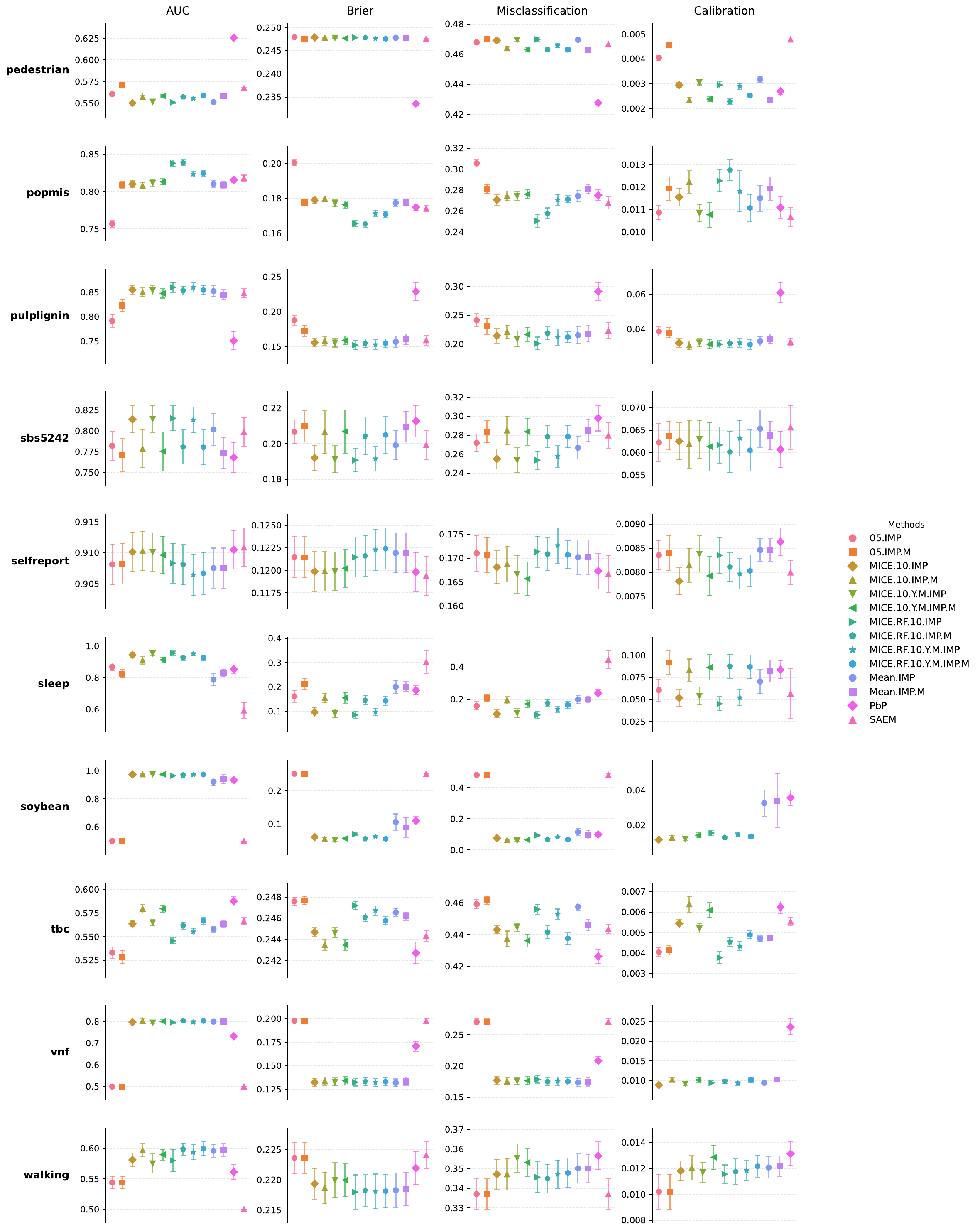}
\vspace{-0.1cm}
\caption{Results on real-world datasets (pedestrian to walking). Performances are evaluated on ROC-AUC, Brier score, Misclassification rate, and Calibration error (via CORP-MCB). Points and error bars represent the mean and standard error across $K=15$ Monte Carlo folds.}
\label{fig:real_data2}
\end{figure*}

\begin{table}[ht]
\centering
\caption{Average AUC results over the 15 Monte-Carlo folds, across real-world datasets. \textbf{Bold} values indicate the best method or those not significantly different (paired t-test, $\alpha = 0.05$). The names of the datasets are abbreviated.}
\label{tab:auc_results}

\begin{tabular}{lllllllllll}
  \hline
Method & air & boy & col & deb & dia & glo & hou & nha & oce & ozo \\ 
  \hline
05.IMP & \textbf{0.67} & 0.99 & 0.71 & 0.61 & 0.84 & \textbf{0.70} & 0.50 & 0.91 & 0.79 & \textbf{0.77} \\ 
  05.IMP.M & 0.64 & \textbf{0.99} & \textbf{0.77} & 0.61 & 0.84 & \textbf{0.70} & 0.50 & 0.94 & 0.84 & \textbf{0.78} \\ 
  MICE.10.IMP & \textbf{0.67} & 0.99 & \textbf{0.78} & \textbf{0.68} & \textbf{0.84} & \textbf{0.70} & \textbf{0.99} & 0.78 & 0.81 & 0.73 \\ 
  MICE.10.IMP.M & \textbf{0.66} & 0.99 & 0.74 & \textbf{0.69} & \textbf{0.84} & \textbf{0.70} & 0.98 & 0.67 & 0.84 & 0.73 \\ 
  MICE.10.Y.M.IMP & \textbf{0.67} & 0.99 & \textbf{0.77} & \textbf{0.68} & 0.84 & \textbf{0.70} & \textbf{0.99} & 0.50 & 0.81 & 0.73 \\ 
  MICE.10.Y.M.IMP.M & \textbf{0.66} & 0.99 & 0.72 & \textbf{0.69} & 0.84 & \textbf{0.70} & 0.98 & 0.50 & 0.84 & 0.73 \\ 
  MICE.RF.10.IMP & \textbf{0.67} & 0.99 & \textbf{0.77} & \textbf{0.68} & \textbf{0.84} & \textbf{0.70} & \textbf{0.99} & 0.86 & 0.81 & 0.72 \\ 
  MICE.RF.10.IMP.M & \textbf{0.66} & 0.99 & 0.72 & \textbf{0.69} & \textbf{0.84} & \textbf{0.70} & 0.99 & 0.68 & 0.84 & 0.72 \\ 
  MICE.RF.10.Y.M.IMP & \textbf{0.68} & 0.99 & \textbf{0.77} & \textbf{0.68} & \textbf{0.84} & \textbf{0.70} & 0.99 & 0.75 & 0.81 & 0.72 \\ 
  MICE.RF.10.Y.M.IMP.M & \textbf{0.65} & 0.99 & 0.72 & \textbf{0.69} & 0.84 & \textbf{0.70} & 0.98 & 0.76 & 0.84 & 0.72 \\ 
  Mean.IMP & \textbf{0.68} & 0.99 & 0.73 & \textbf{0.68} & 0.84 & \textbf{0.70} & 0.99 & 0.95 & 0.82 & 0.72 \\ 
  Mean.IMP.M & 0.64 & 0.98 & 0.69 & \textbf{0.68} & 0.84 & \textbf{0.70} & 0.96 & \textbf{0.97} & 0.84 & 0.72 \\ 
  PbP & 0.64 & 0.98 & 0.55 & 0.64 & 0.82 & \textbf{0.70} & 0.94 & 0.62 & \textbf{0.85} & 0.65 \\ 
  SAEM & \textbf{0.67} & \textbf{0.99} & \textbf{0.76} & 0.62 & 0.84 & 0.69 & 0.50 & 0.50 & 0.81 & \textbf{0.78} \\ 
   \hline
\end{tabular}

\vspace{0.5cm}

\begin{tabular}{lllllllllll}
  \hline
Method & ped & pop & pul & sbs & sel & sle & soy & tbc & vnf & wal \\ 
  \hline
05.IMP & 0.56 & 0.76 & 0.79 & 0.78 & 0.91 & 0.87 & 0.50 & 0.53 & 0.50 & 0.54 \\ 
  05.IMP.M & 0.57 & 0.81 & 0.82 & 0.77 & 0.91 & 0.82 & 0.50 & 0.53 & 0.50 & 0.54 \\ 
  MICE.10.IMP & 0.55 & 0.81 & \textbf{0.86} & \textbf{0.81} & 0.91 & \textbf{0.94} & \textbf{0.97} & 0.56 & \textbf{0.80} & 0.58 \\ 
  MICE.10.IMP.M & 0.56 & 0.81 & 0.85 & 0.78 & 0.91 & 0.91 & \textbf{0.97} & 0.58 & \textbf{0.80} & \textbf{0.60} \\ 
  MICE.10.Y.M.IMP & 0.55 & 0.81 & \textbf{0.85} & \textbf{0.81} & \textbf{0.91} & \textbf{0.95} & \textbf{0.98} & 0.56 & 0.80 & 0.58 \\ 
  MICE.10.Y.M.IMP.M & 0.56 & 0.81 & 0.85 & 0.78 & 0.91 & 0.91 & \textbf{0.97} & 0.58 & 0.80 & 0.59 \\ 
  MICE.RF.10.IMP & 0.55 & \textbf{0.84} & \textbf{0.86} & \textbf{0.82} & 0.91 & \textbf{0.96} & 0.96 & 0.55 & \textbf{0.80} & \textbf{0.58} \\ 
  MICE.RF.10.IMP.M & 0.56 & \textbf{0.84} & \textbf{0.85} & 0.78 & 0.91 & 0.93 & 0.97 & 0.56 & \textbf{0.80} & \textbf{0.60} \\ 
  MICE.RF.10.Y.M.IMP & 0.56 & 0.82 & \textbf{0.86} & \textbf{0.81} & 0.91 & \textbf{0.95} & \textbf{0.97} & 0.56 & \textbf{0.80} & \textbf{0.59} \\ 
  MICE.RF.10.Y.M.IMP.M & 0.56 & 0.82 & \textbf{0.85} & 0.78 & 0.91 & 0.93 & \textbf{0.97} & 0.57 & \textbf{0.80} & \textbf{0.60} \\ 
  Mean.IMP & 0.55 & 0.81 & 0.85 & \textbf{0.80} & 0.91 & 0.79 & 0.92 & 0.56 & \textbf{0.80} & \textbf{0.60} \\ 
  Mean.IMP.M & 0.56 & 0.81 & 0.84 & 0.77 & 0.91 & 0.83 & \textbf{0.94} & 0.56 & 0.80 & \textbf{0.60} \\ 
  PbP & \textbf{0.63} & 0.82 & 0.75 & 0.77 & \textbf{0.91} & 0.85 & 0.93 & \textbf{0.59} & 0.73 & 0.56 \\ 
  SAEM & 0.57 & 0.82 & 0.85 & 0.80 & \textbf{0.91} & 0.59 & 0.50 & 0.57 & 0.50 & 0.50 \\ 
   \hline
\end{tabular}

\end{table}

\clearpage
\section{Miscellaneous}

\begin{figure*}[ht!]
\centering
\includegraphics[width=0.75\textwidth]{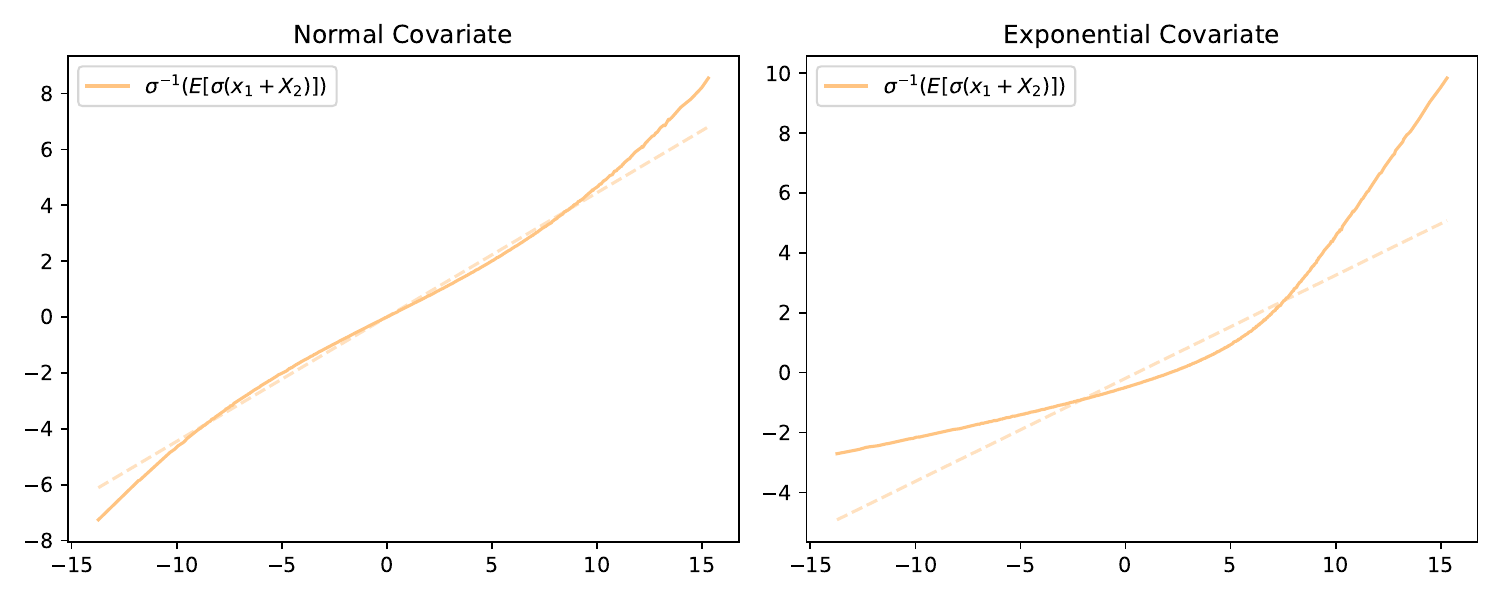}
\caption{Logit transformation of the Bayes probabilities from the illustration in \Cref{sec:illustration_th}, together with a linear approximation. We observe that the Bayes logits are not linear, confirming the theory from \citet{lobo2024primerlinearclassificationmissing}.}
\label{fig:logit_normal_exponential}
\end{figure*}

\end{document}